\documentclass[11pt,reqno]{amsart}
\usepackage{todonotes}
\usepackage{fullpage,times,graphicx,amssymb,amsmath,amsfonts,bbm,psfrag,xcolor,pdfsync}
\usepackage[colorlinks,citecolor=bbluegray,linkcolor=ddarkbrown,urlcolor=blue,breaklinks]{hyperref}
\usepackage{thmtools,thm-restate}
\newtheorem{theorem}{Theorem}[section]
\newtheorem{proposition}[theorem]{Proposition}
\newtheorem{definition}[theorem]{Definition}
\newtheorem{lemma}[theorem]{Lemma}
\newtheorem{corollary}[theorem]{Corollary}

\renewenvironment{proof}{\textbf{Proof.}}{\QED\bigskip}

\usepackage{algorithmic,algorithm}

\usepackage[square]{natbib}

\definecolor{ddarkbrown}{rgb}{0.5,0.2,0.05} \definecolor{bbluegray}{rgb}{0.05,0,0.5}

\newcommand{\BEAS}{\begin{eqnarray*}}
\newcommand{\EEAS}{\end{eqnarray*}}
\newcommand{\BEA}{\begin{eqnarray}}
\newcommand{\EEA}{\end{eqnarray}}
\newcommand{\BEQ}{\begin{equation}}
\newcommand{\EEQ}{\end{equation}}
\newcommand{\BIT}{\begin{itemize}}
\newcommand{\EIT}{\end{itemize}}
\newcommand{\BNUM}{\begin{enumerate}}
\newcommand{\ENUM}{\end{enumerate}}

\newcommand{\BA}{\begin{array}}
\newcommand{\EA}{\end{array}}

\newcommand{\cf}{{\it cf.}}

\newcommand{\ones}{\mathbf 1}

\newcommand{\reals}{{\mathbb R}}

\newcommand{\symm}{{\mbox{\bf S}}}  

\newcommand{\Tr}{\mathop{\bf Tr}}
\newcommand{\diag}{\mathop{\bf diag}}

\newcommand{\idm}{\mathbf{I}}

\newcommand{\Expect}{\textstyle\mathop{\bf E}}

\newcommand{\Prob}{\mathop{\bf Prob}}

\newcommand{\QED}{~~\rule[-1pt]{6pt}{6pt}}

\newcommand{\var}{\mathop{\bf var}}

\usepackage[off]{auto-pst-pdf} 

\title{Spectral Ranking using Seriation}

\author{Fajwel Fogel}
\address{C.M.A.P., \'Ecole Polytechnique,\vskip 0ex
Palaiseau, France.}
\email{fajwel.fogel@cmap.polytechnique.fr}

\author{Alexandre d'Aspremont} 
\address{CNRS \& D.I., UMR 8548, \vskip 0ex
\'Ecole Normale Sup\'erieure, Paris, France.}
\email{aspremon@ens.fr}

\author{Milan Vojnovic} 
\address{Microsoft Research, \vskip 0ex
Cambridge, UK.}
\email{milanv@microsoft.com}

\date{\today}

\begin{document}

\maketitle

\begin{abstract}
We describe a seriation algorithm for ranking a set of items given pairwise comparisons between these items. Intuitively, the algorithm assigns similar rankings to items that compare similarly with all others. It does so by constructing a similarity matrix from pairwise comparisons, using seriation methods to reorder this matrix and construct a ranking. We first show that this spectral seriation algorithm recovers the true ranking when all pairwise comparisons are observed and consistent with a total order. We then show that ranking reconstruction is still exact when some pairwise comparisons are corrupted or missing, and that seriation based spectral ranking is more robust to noise than classical scoring methods. Finally, we bound the ranking error when only a random subset of the comparions are observed. An additional benefit of the seriation formulation is that it allows us to solve semi-supervised ranking problems. Experiments on both synthetic and real datasets demonstrate that seriation based spectral ranking achieves competitive and in some cases superior performance compared to classical ranking methods.
\end{abstract}

\section{Introduction}
We study the problem of ranking a set of $n$ items given pairwise comparisons between these items\footnote{A subset of these results appeared at NIPS 2014.}. The problem of aggregating binary relations has been formulated more than two centuries ago, in the context of emerging social sciences and voting theories \citep{Borda,Condorcet}. The setting we study here goes back at least to \citep{zermelo1929berechnung,Kend40} and seeks to reconstruct a ranking of items from pairwise comparisons reflecting a total ordering. In this case, the directed graph of all pairwise comparisons, where every pair of vertices is connected by exactly one of two possible directed edges, is usually called a {\em tournament} graph in the theoretical computer science literature or a ``round robin'' in sports, where every player plays every other player once and each preference marks victory or defeat. The motivation for this formulation often stems from the fact that in many applications, e.g. music, images, and movies, preferences are easier to express in relative terms (e.g. $a$ is better than $b$) rather than absolute ones (e.g. $a$ should be ranked fourth, and $b$ seventh). In practice, the information about pairwise comparisons is usually {\em incomplete}, especially in the case of a large set of items, and the data may also be {\em noisy}, that is some pairwise comparisons could be incorrectly measured and inconsistent with a total order. 

Ranking is a classical problem but its formulations vary widely. In particular, assumptions about how the pairwise preference information is obtained vary a lot from one reference to another. A subset of preferences is measured adaptively in \citep{Ailo11,Jami11}, while \citep{Freu03,Nega12} extract them at random. In other settings, the full preference matrix is observed, but is perturbed by noise: in e.g. \citep{Brad52,Luce59,Herb06}, a parametric model is assumed over the set of permutations, which reformulates ranking as a maximum likelihood problem. 

Loss functions, performance metrics and algorithmic approaches vary as well. \citet{Keny07}, for example, derive a PTAS for the minimum feedback arc set problem on tournaments, i.e. the problem of finding a ranking that minimizes the number of upsets (a pair of players where the player ranked lower on the ranking beats the player ranked higher). In practice, the complexity of this method is relatively high, and other authors \citep[see e.g.][]{Keen93a,Nega12} have been using spectral methods to produce more efficient algorithms (each pairwise comparison is understood as a link pointing to the preferred item). In other cases, such as the classical Analytic Hierarchy Process (AHP) \citep{Saat80,Barb86} preference information is encoded in a ``reciprocal'' matrix whose Perron-Frobenius eigenvector provides the global ranking. Simple scoring methods such as the point difference rule \citep{Hube63,Waut13} produce efficient estimates at very low computational cost. Website ranking methods such as PageRank~\citep{Page98} and HITS~\citep{Klei99} seek to rank web pages based on the hyperlink structure of the web, where links do not necessarily express consistent preference relationships (e.g.~$a$ can link to $b$ and $b$ can link $c$, and $c$ can link to $a$). \citep{Nega12} adapt the PageRank argument to the ranking from pairwise comparisons and \citet{vigna2009spectral} provides a review of ranking algorithms given pairwise comparisons, in particular those involving the estimation of the stationary distribution of a Markov chain. Ranking has also been approached as a prediction problem, i.e. learning to rank \citep{Scha98,Rajk14}, with \citep{Joac02a} for example using support vector machines to learn a score function. Finally, in the Bradley-Terry-Luce framework, where multiple observations on pairwise preferences are observed and assumed to be generated by a generalized linear model, the maximum likelihood problem is usually solved using fixed point algorithms or EM-like majorization-minimization techniques \citep{Hunt04}. \citet{jiang2011statistical} describes the HodgeRank algorithm, which formulates ranking given pairwise comparisons as a least-square problem. This formulation is based on Hodge theory and provides tools to measure the consistency of a set of pairwise comparisons with the existence of a global ranking.   \citet{duchi2010consistency,duchi2013asymptotics} analyze the consistency of various ranking algorithms given pairwise comparisons and a query. Preferences are aggregated through standard procedures, e.g., computing the mean of comparisons from different users, then ranking are derived using classical algorithms, e.g., Borda Count, Bradley-Terry-Model maximum likelihood estimation, least squares, odd-ratios \citep{saaty2003decision}.

Here, we show that the ranking problem is directly related to another classical ordering problem, namely {\em seriation}. Given a similarity matrix between a set of $n$ items and assuming that the items can be ordered along a chain (path) such that the similarity between items decreases with their distance within this chain (i.e. a total order exists), the seriation problem seeks to reconstruct the underlying linear ordering based on unsorted, possibly noisy, pairwise similarity information. \citet{Atki98} produced a spectral algorithm that exactly solves the seriation problem in the noiseless case, by showing that for similarity matrices computed from serial variables, the ordering of the eigenvector corresponding to the second smallest eigenvalue of the Laplacian matrix (a.k.a.~the Fiedler vector) matches that of the variables. In practice, this means that performing spectral ordering on the similarity matrix exactly reconstructs the correct ordering provided items are organized in a chain.

We adapt these results to ranking to produce a very efficient {\em spectral ranking algorithm with provable recovery and robustness guarantees}. Furthermore, the seriation formulation allows us to handle semi-supervised ranking problems. \cite{Foge13} show that seriation is equivalent to the 2-SUM problem and study convex relaxations to seriation in a semi-supervised setting, where additional structural constraints are imposed on the solution. Several authors \citep{Blum00,Feig07} have also focused on the directly related Minimum Linear Arrangement (MLA) problem, for which excellent approximation guarantees exist in the noisy case, albeit with very high polynomial complexity.

The main contributions of this paper can be summarized as follows. We link seriation and ranking by showing how to construct a consistent similarity matrix based on consistent pairwise comparisons. We then recover the true ranking by applying the spectral seriation algorithm in \citep{Atki98} to this similarity matrix (we call this method \ref{alg:SerialRank} in what follows). In the noisy case, we then show that spectral seriation can perfectly recover the true ranking even when some of the pairwise comparisons are either corrupted or missing, provided that the pattern of errors is somewhat unstructured. We show in particular that, in a regime where a high proportion of comparisons are observed, some incorrectly, the spectral solution is more robust to noise than classical scoring based methods. On the other hand, when only few comparisons are observed, we show that for Erd\"os-R\'enyi graphs, i.e., when pairwise comparisons are observed independently with a given probability, 
$\Omega(n \log^4 n)$ comparisons suffice for $\ell_2$ consistency of the Fiedler vector and hence $\ell_2$ consistency of the retreived ranking w.h.p. On the other hand we need $\Omega(n^{3/2} \log^4 n)$ comparisons to retrieve a ranking whose local perturbations are bounded in $\ell_\infty$ norm.
Since for Erd\"os-R\'enyi graphs the induced graph of comparisons is connected with high probability only when the total number of pairs sampled scales as $\Omega(n \log n)$ (aka the coupon collector effect), we need at least that many comparisons in order to retrieve a ranking, therefore the $\ell_2$ consistency result can be seen as optimal up to a polylogarithmic factor. 
Finally, we use the seriation results in \citep{Foge13} to produce semi-supervised ranking solutions.

The paper is organized as follows. In Section~\ref{s:seriation} we recall definitions related to seriation, and link ranking and seriation by showing how to construct well ordered similarity matrices from well ranked items.  In Section~\ref{s:spectral} we apply the spectral algorithm of \citep{Atki98} to reorder these similarity matrices and reconstruct the true ranking in the noiseless case. In Section~\ref{s:robust} we then show that this spectral solution remains exact in a noisy regime where a random subset of comparisons is corrupted. In Section~\ref{s:spectralPertAnalysis} we analyze ranking perturbation results when only few comparisons are given following an Erd\"os-R\'enyi graph. Finally, in Section~\ref{s:numres} we illustrate our results on both synthetic and real datasets, and compare ranking performance with classical MLE, spectral and scoring based approaches.

\section{Seriation, Similarities \& Ranking} \label{s:seriation}
In this section we first introduce the seriation problem, i.e. reordering items based on pairwise similarities. We then show how to write the problem of ranking given pairwise comparisons as a seriation problem.

\subsection{The Seriation Problem}
The seriation problem seeks to reorder $n$ items given a similarity matrix between these items, such that the more similar two items are, the closer they should be. This is equivalent to supposing that items can be placed on a chain where the similarity between two items decreases with the distance between these items in the chain. We formalize this below, following~\citep{Atki98}.

\begin{definition}\label{def:r-mat}
We say that a matrix $A \in\symm_n$ is an R-matrix (or Robinson matrix) if and only if it is symmetric and $A_{i,j} \leq  A_{i,j+1}$ and $A_{i+1,j} \leq A_{i,j}$ in the lower triangle, where $1\leq j < i \leq n$.
\end{definition}

Another way to formulate R-matrix conditions is to impose $A_{ij}\geq A_{kl}$ if $|i-j|\leq |k-l|$ off-diagonal, i.e.~the coefficients of $A$ decrease as we move away from the diagonal. We also introduce a definition for strict R-matrices $A$, whose rows and columns cannot be permuted without breaking the R-matrix monotonicity conditions. We call \emph{reverse identity} permutation the permutation that puts rows and columns $1,2,\ldots,n$ of a matrix $A$ in reverse order $n,n-1,\ldots,1$.

\begin{definition}
\label{def:strictR}
An R-matrix $A\in\symm_n$ is called strict-R if and only if the identity and reverse identity permutations of $A$ are the only permutations reordering $A$ as an R-matrix.
\end{definition}

Any R-matrix with only strict R-constraints is a strict R-matrix. Following \citep{Atki98}, we will say that $A$ is {\em pre-R} if there is a permutation matrix $\Pi$ such that $\Pi A \Pi^T$ is an R-matrix. Given a pre-R matrix $A$, the seriation problem consists in finding a permutation $\Pi$ such that $\Pi A \Pi^T$ is an R-matrix. Note that there might be several solutions to this problem. In particular, if a permutation $\Pi$ is a solution, then the reverse permutation is also a solution. When only two permutations of $A$ produce R-matrices, $A$ will be called {\em pre-strict-R}. 

\subsection{Constructing Similarity Matrices from Pairwise Comparisons} 
\label{sec:simMatrix}

Given an ordered input pairwise comparison matrix, we now show how to construct a similarity matrix which is {\em strict-R} when all comparisons are given and consistent with the identity ranking (i.e., items are ranked in increasing order of indices). This means that the similarity between two items decreases with the distance between their ranks. We will then be able to use the spectral seriation algorithm by \citep{Atki98} described in Section~\ref{s:spectral} to reconstruct the true ranking from a disordered similarity matrix.

We first show how to compute a pairwise similarity from pairwise comparisons between items by counting the number of matching comparisons. Another formulation allows us to handle the generalized linear model. These two examples are only two particular instances of a broader class of ranking algorithms derived here. Any method which produces R-matrices from pairwise preferences yields a valid ranking algorithm. 

\subsubsection{Similarities from Pairwise Comparisons}
Suppose we are given a matrix of pairwise comparisons $C \in \{-1,0,1\}^{n\times n}$ such that $C_{i,j} = -C_{j,i}$ for every $i\neq j$ and
\BEQ \label{eq:c-mat}
C_{i,j} = 
\left\{\BA{cl}
1  & \mbox{if } i \mbox{ is ranked higher than } j\\
0 & \mbox{if $i$ and $j$ are not compared or in a draw}\\
-1 & \mbox{if } j \mbox{ is ranked higher than } i
\EA\right.
\EEQ
setting $C_{i,i} = 1$ for all $i \in \{1,\ldots,n\}$. We define the pairwise similarity matrix $S^\mathrm{match}$ as
\BEQ\label{eq:simMatchDef}
S^\mathrm{match}_{i,j}=\sum_{k=1}^n\left( \frac{1+C_{i,k}C_{j,k}}{2}\right).
\EEQ
Since $C_{i,k}C_{j,k}=1$, if $C_{i,k}$ and $C_{j,k}$ have matching signs, and $C_{i,k}C_{j,k}=-1$ if they have opposite signs, $S^\mathrm{match}_{i,j}$ counts the number of matching comparisons between $i$ and $j$ with other reference items $k$. If $i$ or $j$ is not compared with $k$, then $C_{i,k}C_{j,k}=0$ and the term $ (1+C_{i,k}C_{j,k})/2$ has an neutral effect on the similarity of  $1/2$. Note that we also have 
\BEQ\label{eq:rnk-sim}
S^\mathrm{match}=\frac{1}{2} \left(n \ones \ones^T+ CC^T\right).
\EEQ
The intuition behind the similarity $S^\mathrm{match}$ is easy to understand in a tournament setting: players that beat the same players and are beaten by the same players should have a similar ranking. 

The next result shows that when all comparisons are given and consistent with the identity ranking, then the similarity matrix $S^\mathrm{match}$ is a strict R-matrix. Without loss of generality, we assume that items are ranked in increasing order of their indices. In the general case, we can simply replace the \emph{strict-R} property by the \emph{pre-strict-R} property.

\begin{proposition}\label{prop:s-match-R}
Given all pairwise comparisons between items ranked according to the identity permutation (with no ties), the similarity matrix $S^\mathrm{match}$ constructed in~\eqref{eq:simMatchDef} is a strict R-matrix and
\BEQ\label{eq:s-match-ideal}
S^\mathrm{match}_{i,j} =  n - |i-j|
\EEQ
for all $i,j=1,\ldots,n$.
\end{proposition}

\begin{proof}
Since items are ranked as $1,2,\ldots,n$ with no ties and all comparisons given, $C_{i,j}=-1$ if $i<j$ and $C_{i,j}=1$ otherwise. 
Therefore we obtain from definition~\eqref{eq:simMatchDef}
\begin{eqnarray*}
S^\mathrm{match}_{i,j}  
&=& \sum_{k=1}^{\min(i,j)-1} \left( \frac{1+1}{2}\right) + \sum_{k=\min(i,j)}^{\max(i,j)-1} \left( \frac{1-1}{2}\right) + \sum_{k=\max(i,j)}^{n} \left( \frac{1+1}{2}\right) \\
&=& n - (\max(i,j)-\min(i,j))\\
&=& n - |i-j|
\end{eqnarray*}
This means in particular that $S^\mathrm{match}$ is strictly positive and its coefficients are strictly decreasing when moving away from the diagonal, hence $S^\mathrm{match}$ is a strict R-matrix. 
\end{proof}

\subsubsection{Similarities in the Generalized Linear Model}
\label{s:gln}
Suppose that paired comparisons are generated according to a generalized linear model (GLM), i.e., we assume that the outcomes of paired comparisons are independent and for any pair of distinct items, item $i$ is observed ranked higher than item $j$ with probability 
\BEQ\label{eq:p-glm}
P_{i,j}= H(\nu_i-\nu_j)
\EEQ
where $\nu\in\reals^n$ is a vector of skill parameters and $H : \reals \rightarrow [0,1]$ is a function that is increasing on $\reals$ and such that $H(-x) = 1 - H(x)$ for all $x\in \reals$, and $\lim_{x\rightarrow -\infty}H(x) = 0$ and $\lim_{x\rightarrow \infty}H(x) = 1$. A well known special instance of the generalized linear model is the Bradley-Terry-Luce model for which $H(x) = 1/(1+e^{-x})$, for $x\in \reals$.

Let $m_{i,j}$ be the number of times items $i$ and $j$ were compared, $C_{i,j}^s \in \{-1,1\}$ be the outcome of comparison $s$ and $Q$ be the matrix of corresponding sample probabilities, i.e. if $m_{i,j} > 0$ we have
$$
Q_{i,j}= \frac{1}{m_{i,j}} \sum_{s=1}^{m_{i,j}}  \frac{C_{i,j}^s+1}{2}
$$
and $Q_{i,j} = 1/2$ in case $m_{i,j} = 0$. We define the similarity matrix $S^\mathrm{glm}$ from the observations $Q$ as
\BEQ\label{eq:s-glm}
S^\mathrm{glm}_{i,j}= \sum_{k=1}^n \ones_{\{m_{i,k}m_{j,k}>0\}}\left(1 - | Q_{i,k}- Q_{j,k}| \right)+ \frac{\ones_{\{m_{i,k}m_{j,k}=0\}}}{2}.
\EEQ
Since the comparison observations are independent we have that $Q_{i,j}$ converges to $P_{i,j}$ as $m_{i,j}$ goes to infinity and the central limit theorem implies that $S^\mathrm{glm}_{i,j}$ converges to a Gaussian variable with mean 
\[
\sum_{k=1}^n \left(1 - | P_{i,k}- P_{j,k}|\right).
\] 
The result below shows that this limit similarity matrix is a strict R-matrix when items are properly ordered.

\begin{proposition}\label{prop:RmatGLM}
If items are ordered according to the order in decreasing values of the skill parameters, the similarity matrix $S^\mathrm{glm}$ is a strict $R$ matrix with high probability as the number of observations goes to infinity.
\end{proposition}
\begin{proof}
Without loss of generality, we suppose the true order is $1,2,\ldots,n$, with $\nu(1)>\ldots>\nu(n)$. For any $i,j,k$ such that $i>j$, using the GLM assumption (i) we get
\[
P_{i,k} =H(\nu(i)-\nu(k)) <  H(\nu(j)-\nu(k)) = P_{j,k}.
\]
Since empirical probabilities $Q_{i,j}$ converge to $P_{i,j}$, when the number of observations is large enough, we also have $Q_{i,k} <  Q_{j,k}$ for any $i,j,k$ such that $i>j$ (we focus w.l.o.g. on the lower triangle), and we can therefore remove the absolute value in the expression of $S^\mathrm{glm}_{i,j}$ for $i>j$.
Hence for any $i>j$ we have
\BEAS
S^\mathrm{glm}_{i+1,j}-S^\mathrm{glm}_{i,j} & = &   -\sum_{k=1}^n | Q_{i+1,k}- Q_{j,k}| +  \sum_{k=1}^n | Q_{i,k}- Q_{j,k}| \\
& = &    \sum_{k=1}^n (Q_{i+1,k}- Q_{j,k}) - (Q_{i,k}- Q_{j,k}) \\
& = &   \sum_{k=1}^n  Q_{i+1,k}- Q_{i,k}  < 0.
\EEAS
Similarly for any $i>j$, $S^\mathrm{glm}_{i,j-1}-S^\mathrm{glm}_{i,j} < 0$, so $S^\mathrm{glm}$ is a strict R-matrix.
 \end{proof}
 
Notice that we recover the original definition of $S^\mathrm{match}$ in the case of binary comparisons, though it does not fit in the Generalized Linear Model. Note also that these definitions can be directly extended to the setting where multiple comparisons are available for each pair and aggregated in comparisons that take fractional values (e.g., a tournament setting where participants play several times against each other).

\section{Spectral Algorithms}\label{s:spectral}
We first recall how spectral ordering can be used to recover the true ordering in seriation problems. We then apply this method to the ranking problem.
 
\subsection{Spectral Seriation Algorithm} \label{sec:spectralSer} 
We use the spectral computation method originally introduced in \citep{Atki98} to solve the seriation problem based on the similarity matrices defined in the previous section. We first recall the definition of the Fiedler vector (which is shown to be unique in our setting in Lemma~\ref{lemma:FiedlerSimple}).

\begin{definition}\label{def:Fiedler} The Fiedler value of a symmetric, nonnegative and irreducible matrix $A$ is the smallest non-zero eigenvalue of its Laplacian matrix $L_A=\diag(A\ones)-A$. The corresponding eigenvector is called Fiedler vector and is the optimal solution to $\min\{y^T  L_A  y : y\in \reals^n, y^T\ones=0, \|y\|_2=1\}$.
\end{definition}

The main result from \citep{Atki98}, detailed below, shows how to reorder pre-R matrices in a noise free case.  

\begin{proposition}\cite[Th.\,3.3]{Atki98}\label{theo:AtkinsFiedler} Let $A\in\symm_n$ be an irreducible pre-R-matrix with a simple Fiedler value and a Fiedler vector $v$ with no repeated values. Let $\Pi_1 \in \mathcal{P}$ (respectively, $\Pi_2$) be the permutation such that the permuted Fiedler vector $\Pi_1 v$ is strictly increasing (decreasing). Then $\Pi_1  A \Pi_1^T$ and $\Pi_2  A \Pi_2^T$ are R-matrices, and no other permutations of A produce R-matrices.
\end{proposition}

The next technical lemmas extend the results in \cite{Atki98} to strict R-matrices and will be used to prove Theorem~\ref{thm:exactRecoverySerialRank} in next section. The first one shows that without loss of generality, the Fiedler value is simple.

\begin{lemma}\label{lemma:FiedlerSimple}
If $A$ is an irreducible R-matrix, up to a uniform shift of its coefficients, $A$ has a simple Fiedler value and a monotonic Fiedler vector.
\end{lemma}
\begin{proof}
We use \cite[Th.\,4.6]{Atki98} which states that if $A$ is an irreducible R-matrix with $A_{n,1} = 0$, then the Fiedler
value of $A$ is a simple eigenvalue.
Since $A$ is an R-matrix, $A_{n,1}$ is among its minimal elements. Subtracting it from $A$ does not affect the nonnegativity of $A$ and we can apply \cite[Th.\,4.6]{Atki98}.
Monotonicity of the Fiedler vector then follows from \cite[Th.\,3.2]{Atki98}.
\end{proof}

The next lemma shows that the Fiedler vector is strictly monotonic if $A$ is a strict R-matrix.

\begin{lemma}\label{lemma:strictR}
Let $A\in\symm_n$ be an irreducible R-matrix. Suppose there are no distinct indices $r<s$ such that for any $k \not \in [r,s], A_{r,k}=A_{r+1,k}=\ldots=A_{s,k}$, then, up to a uniform shift, the Fiedler value of $A$ is simple and its Fiedler vector is strictly monotonic.
\end{lemma}
\begin{proof}
By Lemma~\ref{lemma:FiedlerSimple}, the Fiedler value of $A$ is simple (up to a uniform shift of $A$). Let $x$ be the corresponding Fiedler vector of $A$, $x$ is monotonic by Lemma~\ref{lemma:FiedlerSimple}. Suppose $[r,s]$ is a nontrivial maximal interval such that $x_r =x_{r+1}=\ldots = x_s$, then by \cite[lemma\,4.3]{Atki98}, for any $k \not \in [r,s], A_{r,k}=A_{r+1,k}=\ldots=A_{s,k}$, which contradicts the initial assumption. Therefore $x$ is strictly monotonic.
\end{proof}

In fact, we only need a small portion of the R-constraints to be strict for the previous lemma to hold. We now show that the main assumption on $A$ in Lemma~\ref{lemma:strictR} is equivalent to A being {strict-R}.
\begin{lemma}
\label{lemma:strictReq}
An irreducible R-matrix $A\in\symm_n$ is strictly $R$ if and only if there are no distinct indices $r<s$ such that for any $k \not \in [r,s], A_{r,k}=A_{r+1,k}=\ldots=A_{s,k}$.
\end{lemma}
\begin{proof}
Let $A\in\symm_n$ an R-matrix.
Let us first suppose there are no distinct indices $r<s$ such that for any $k \not \in [r,s], A_{r,k}=A_{r+1,k}=\ldots=A_{s,k}$. By Lemma~\ref{lemma:strictR} the Fiedler value of $A$ is simple and its Fiedler vector is strictly monotonic. Hence by Proposition~\ref{theo:AtkinsFiedler}, only the identity and reverse identity permutations of $A$ produce R-matrices.
Now suppose there exist two distinct indices $r<s$ such that for any $k \not \in [r,s], A_{r,k}=A_{r+1,k}=\ldots=A_{s,k}$. In addition to the identity and reverse identity permutations, we can locally reverse the order of rows and columns from $r$ to $s$, since the sub matrix $A_{r:s,r:s}$ is an R-matrix and for any $k \not \in [r,s], A_{r,k}=A_{r+1,k}=\ldots=A_{s,k}$. Therefore at least four different permutations of $A$ produce R-matrices, which means that $A$ is not strictly R.
\end{proof}

\subsection{SerialRank: a Spectral Ranking Algorithm}
In Section~\ref{s:seriation}, we showed that similarities $S^\mathrm{match}$ and $S^\mathrm{glm}$ are \emph{pre-strict-R} when all comparisons are available and consistent with an underlying ranking of items. We now use the spectral seriation method in \citep{Atki98} to reorder these matrices and produce a ranking. Spectral ordering requires computing an extremal eigenvector, at a cost of $O(n^2\log n)$ flops \citep{Kucz92}. We call this algorithm~\ref{alg:SerialRank} and prove the following result.

\begin{theorem}
\label{thm:exactRecoverySerialRank}
Given all pairwise comparisons for a set of totally ordered items and assuming there are no ties between items, algorithm~\ref{alg:SerialRank}, i.e., sorting the Fiedler vector of the matrix $S^\mathrm{match}$ defined in~\eqref{eq:rnk-sim}, recovers the true ranking of items.
\end{theorem}

\begin{proof}
From Proposition~\ref{prop:s-match-R}, under assumptions of the proposition $S^\mathrm{match}$ is a pre-strict R-matrix. Now combining the definition of strict-R matrices in Lemma~\ref{lemma:strictReq} with Lemma~\ref{lemma:strictR}, we deduce that Fiedler value of $S^\mathrm{match}$ is simple and its Fiedler vector has no repeated values. Hence by Proposition~\ref{theo:AtkinsFiedler}, only the two permutations that sort the Fiedler vector in increasing and decreasing order produce strict R-matrices and are candidate rankings (by Proposition~\ref{prop:s-match-R} $S^\mathrm{match}$ is a strict R-matrix when ordered according to the true ranking). Finally we can choose between the two candidate rankings (increasing and decreasing) by picking the one with the least upsets. 
\end{proof}

Similar results apply for $S^\mathrm{glm}$ given enough comparisons in the Generalized Linear Model.
This last result guarantees recovery of the true ranking of items in the noiseless case. In the next section, we will study the impact of corrupted or missing comparisons on the inferred ranking of items.

\begin{algorithm}[t]
\caption{\bf (SerialRank)}
\begin{algorithmic} [1]
\REQUIRE A set of pairwise comparisons $C_{i,j} \in \{-1,0,1\}$ or $[-1,1]$.
\STATE Compute a similarity matrix $S$ as in \S\ref{sec:simMatrix}
\STATE Compute the Laplacian matrix 
\BEQ\label{alg:SerialRank}
\tag{SerialRank}
L_S=\diag(S\ones)-S
\EEQ
\STATE Compute the Fiedler vector of $S$.
\ENSURE A ranking induced by sorting the Fiedler vector of $S$ (choose either increasing or decreasing order to minimize the number of upsets).
\end{algorithmic}
\end{algorithm}

\section{Exact Recovery with Corrupted and Missing Comparisons}\label{s:robust}

\begin{figure}[b]
\begin{center}
\begin{tabular}{cccc}
\includegraphics[scale=0.28]{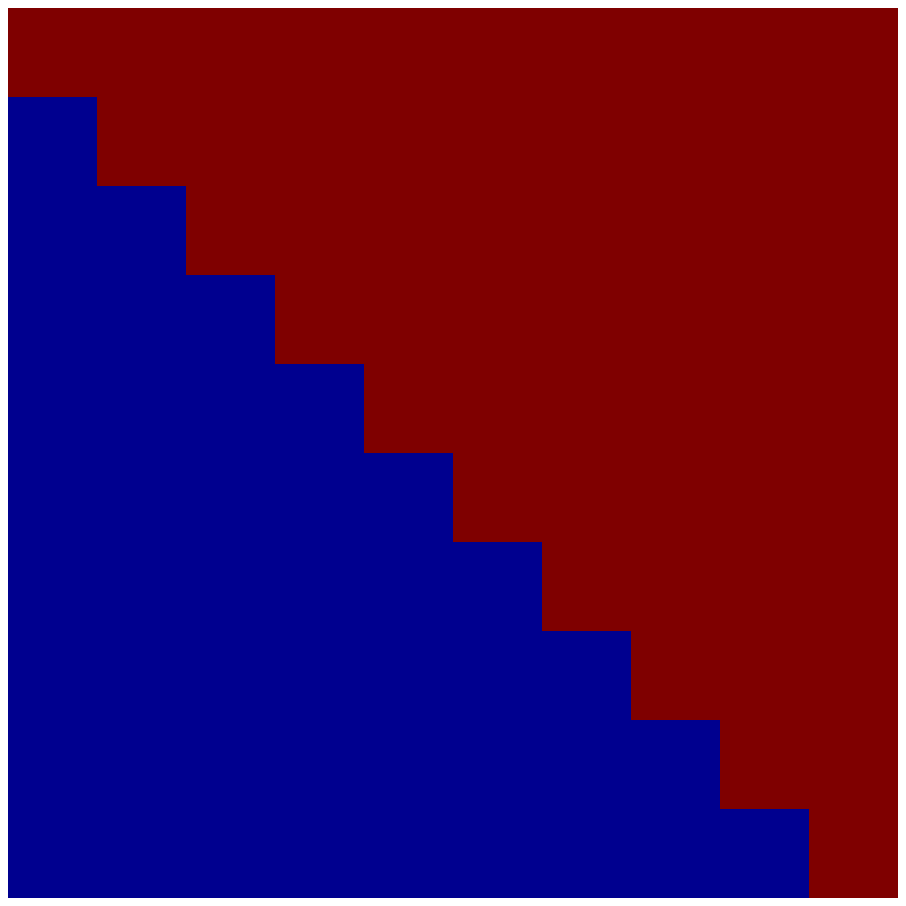}&
\includegraphics[scale=0.28]{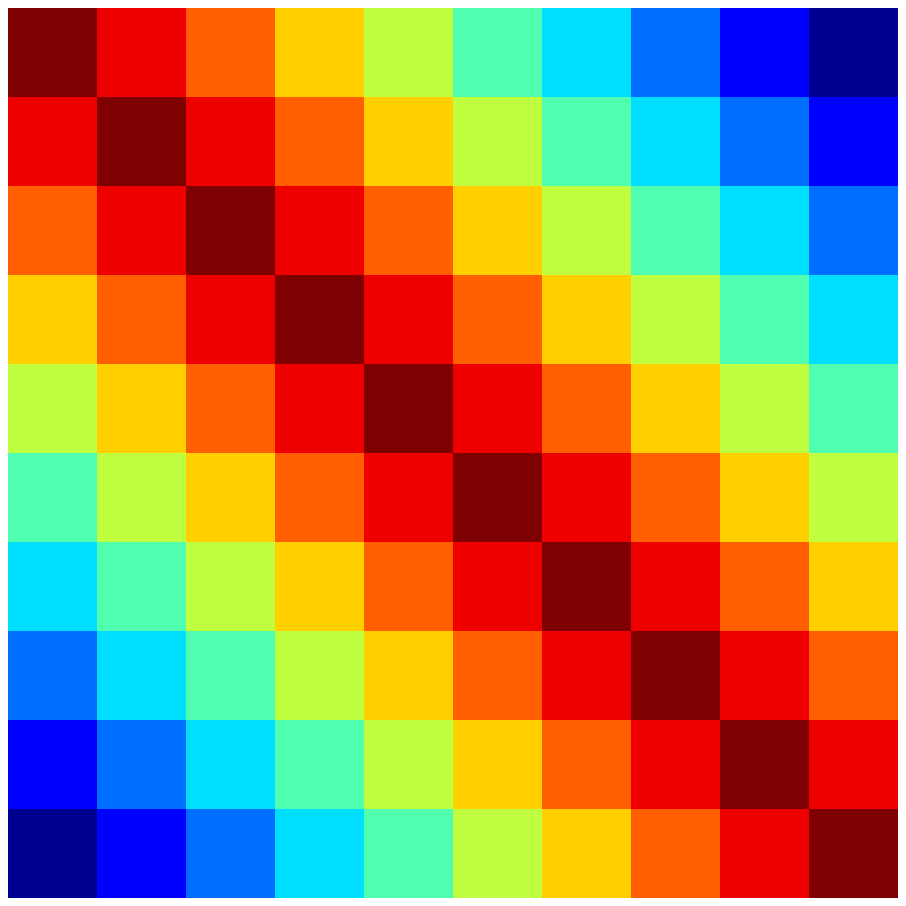}&
\includegraphics[scale=0.28]{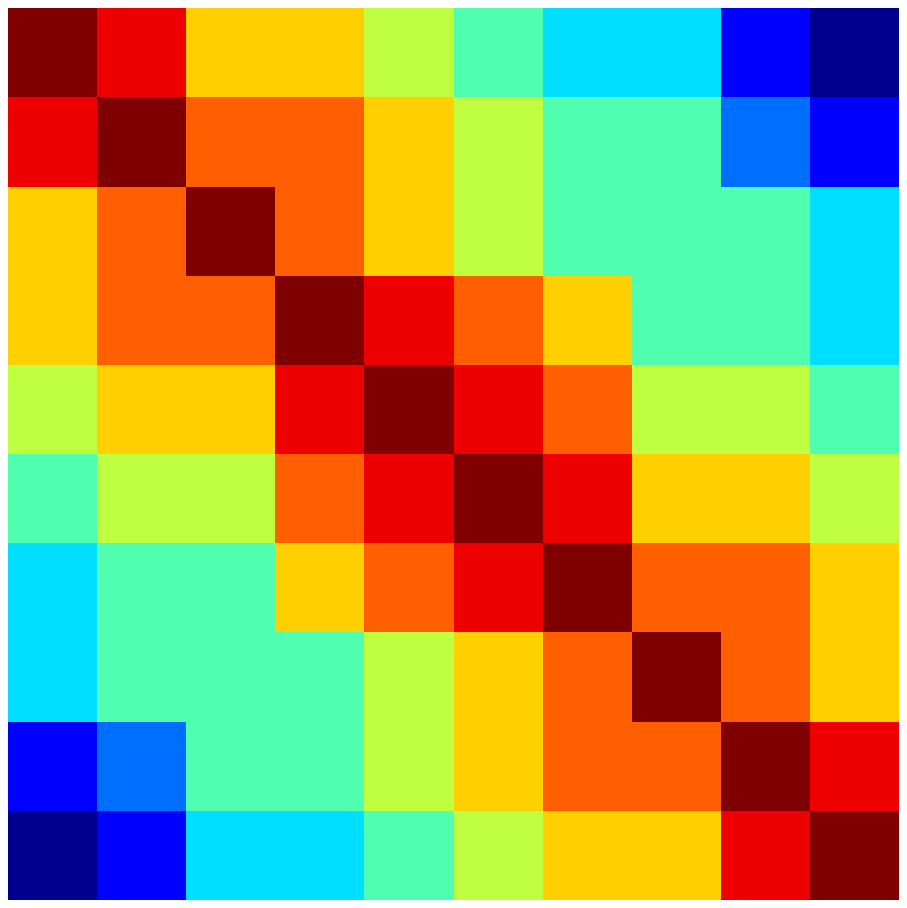}&
\includegraphics[scale=0.4]{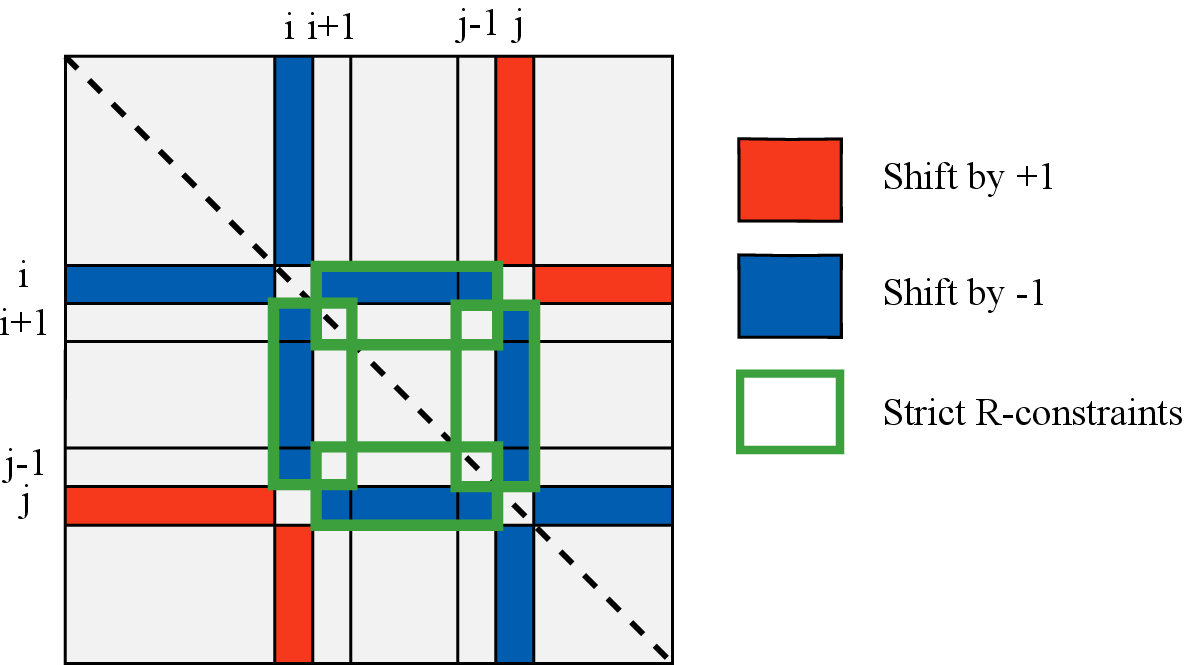}
\end{tabular}
\caption{The matrix of pairwise comparisons $C$ {\em (far left)} when the rows are ordered according to the true ranking. The corresponding similarity matrix $S^\mathrm{match}$ is a strict R-matrix {\em (center left)}. The same $S^{\mathrm{match}}$ similarity matrix with comparison (3,8) corrupted {\em (center right)}. With one corrupted comparison, $S^{\mathrm{match}}$ keeps enough strict R-constraints to recover the right permutation. In the noiseless case, the difference between all coefficients is at least one and after introducing an error, the coefficients inside the green rectangles still enforce strict R-constraints {\em(far right)}.
\label{fig:simMatrixOneError}}
\end{center}
\end{figure}

In this section we study the robustness of \ref{alg:SerialRank} using $S^\mathrm{match}$ with respect to noisy and missing pairwise comparisons. We will see that noisy comparisons cause ranking ambiguities for the point score method and that such ambiguities are to be lifted by the spectral ranking algorithm. We show in particular that the \ref{alg:SerialRank} algorithm recovers the exact ranking when the pattern of errors is random and errors are not too numerous. We first study the impact of one corrupted comparison on \ref{alg:SerialRank}, then extend the result to multiple corrupted comparisons. A similar analysis is provided for missing comparisons as Corollary~\ref{cor:impactCorMiss}. in the Appendix. Finally, Proposition~\ref{prop:proba} provides an estimate of the number of randomly corrupted entries that can be tolerated for perfect recovery of the true ranking. We begin by recalling the definition of the \emph{point score} of an item.
\begin{definition}\label{def:score}
The {\em point score} $w_i$ of an item $i$, also known as point-difference, or {\em row-sum} is defined as
$w_i = \sum_{k=1}^{n} C_{k,i}$,
which corresponds to the number of wins minus the number of losses in a tournament setting.
\end{definition}
In the following we will denote by $w$ the point score vector.


\begin{proposition}\label{prop:sup-oneError}
Given all pairwise comparisons $C_{s,t} \in \{-1,1\}$ between items ranked according to their indices, suppose the sign of one comparison $C_{i,j}$ (and its counterpart $C_{j,i}$) is switched, with $i<j$. If $j-i>2$ then $S^\mathrm{match}$ defined in~\eqref{eq:rnk-sim} remains \emph{strict-R}, whereas the point score vector $w$ has ties between items $i$ and $i+1$ and items $j$ and $j-1$.
\end{proposition}
\begin{proof}
We give some intuition for the result in Figure~\ref{fig:simMatrixOneError}. 
We write the true score and comparison matrix $w$ and $C$, while the observations are written $\hat w$ and $\hat C$ respectively. This means in particular that $\hat C_{i,j}=-C_{i,j}=1$ and $\hat C_{j,i}=-C_{j,i}=-1$. 
To simplify notations we denote by $S$ the similarity matrix $S^\mathrm{match}$ (respectively $\hat S$ when the similarity is computed from observations).
We first study the impact of a corrupted comparison $C_{i,j}$ for $i<j$ on the point score vector $\hat{w}$. We have
\[
\hat{w}_i = \sum_{k=1}^{n} \hat C_{k,i} = \sum_{k=1}^{n} C_{k,i} + \hat C_{j,i} - C_{j,i} = w_i - 2 = w_{i+1},
\]
similarly $\hat{w}_j=w_{j-1}$, whereas $\hat{w_k}=w_k$ for $k\neq i,j$.
Hence, the incorrect comparison induces two ties in the point score vector $w$.  Now we show that the similarity matrix defined in \eqref{eq:rnk-sim} breaks these ties, by showing that it is a strict R-matrix. Writing $\hat S$ in terms of $S$, we get for any $t\neq i,j$
\[
[\hat C\hat C^T]_{i,t}= \sum_{k\neq j} \left( \hat C_{i,k}\hat C_{t,k} \right) + \hat C_{i,j}\hat C_{t,j} 
= \sum_{k\neq j} \left( C_{i,k}C_{t,k} \right) + \hat C_{i,j}C_{t,j} = 
\left\{\BA{ll}
[CC^T]_{i,t} -2  & \mbox{if} \ t  <  j \\
\left[CC^T\right]_{i,t} +2 & \mbox{if} \ t > j.\\
\EA\right.
\]
Thus we obtain
\[ \hat S_{i,t} =
 \left\{\BA{ll}
S_{i,t} - 1  & \mbox{if} \ t  < j \\
S_{i,t} + 1 & \mbox{if} \ t > j,\\
\EA\right. \]
(remember there is a factor $1/2$ in the definition of $S$). Similarly we get for any $t \neq i,j$
\[ \hat S_{j,t} =
 \left\{\BA{ll}
S_{j,t} + 1  & \mbox{if} \ t  <  i \\
S_{j,t} - 1 & \mbox{if} \ t > i.\\
\EA\right. \]
Finally, for the single corrupted index pair $(i,j)$, we get
\[
\hat S_{i,j}= \frac{1}{2} \left( n+ \sum_{k\neq i,j} \left( \hat C_{i,k}\hat C_{j,k} \right) + \hat C_{i,i}\hat C_{j,i} + \hat C_{i,j}\hat C_{j,j} \right)
= S_{i,j}  - 1 +1 = S_{i,j}.
\]
The diagonal of $S$ is not impacted since $[\hat C\hat C^T]_{i,i}= \sum_{k=1}^n \left( \hat C_{i,k}\hat C_{i,k} \right)=n$. For all other coefficients $(s,t)$ such that $s,t \neq i,j$, we also have $\hat S_{s,t}=S_{s,t}$, which means that all rows or columns outside of $i,j$ are left unchanged. We first observe that these last equations, together with our assumption that $j-i>2$ and the fact that the elements of the exact $S$ in~\eqref{eq:s-match-ideal} differ by at least one, imply that
\[
\hat S_{s,t} \leq \hat S_{s+1,t} \quad \mbox{and} \quad \hat S_{s,t+1} \leq \hat S_{s,t},\quad \mbox{for  } s < t
\]
so $\hat S$ remains an R-matrix.
Note that this result remains true even when $j-i=2$, but we need some strict inequalities to show uniqueness of the retrieved order. Indeed, because $j-i>2$ all these $R$ constraints are strict except between elements of rows $i$ and $i+1$, and rows $j-1$ and $j$ (and similarly for columns).
These ties can be broken using the fact that
\[
\hat S_{i,j-1}=S_{i,j-1}-1<S_{i+1,j-1}-1=\hat S_{i+1,j-1}-1<\hat S_{i+1,j-1}
\] 
which means that $\hat S$ is still a strict R-matrix (see Figure \ref{fig:simMatrixOneError}) since $j-1>i+1$ by assumption.
\end{proof}

We now extend this result to multiple errors.

\begin{proposition}\label{prop:sup-kerrors}
Given all pairwise comparisons $C_{s,t} \in \{-1,1\}$ between items ranked according to their indices, suppose the signs of $m$ comparisons indexed $(i_1,j_1),\ldots,(i_m,j_m)$ are switched. If the following condition~\eqref{cond:app-distrkerrors} holds true, 
\BEQ \label{cond:app-distrkerrors} 
| s - t | > 2, \mbox{ for all }  s,t\in \{i_1,\dots,i_{m},j_1,\dots,j_{m}\} \mbox{ with } s\neq t ,
\EEQ 
then $S^\mathrm{match}$ defined in~\eqref{eq:rnk-sim} remains \emph{strict-R}, whereas the point score vector $w$ has $2m$ ties.
\end{proposition}
\begin{proof}
We write the true score and comparison matrix $w$ and $C$, while the observations are written $\hat w$ and $\hat C$ respectively, and without loss of generality we suppose $i_l<j_l$. This implies that $\hat C_{i_l, j_l}=-C_{i_l, j_l}=1$ and $\hat C_{j_l, i_l}=-C_{j_l, i_l}=-1$ 
for all $l$ in $\{1,\ldots,m\}$. To simplify notations, we denote by $S$ the similarity matrix $S^\mathrm{match}$ (respectively $\hat S$ when the similarity is computed from observations).

As in the proof of Proposition~\ref{prop:sup-oneError}, corrupted comparisons indexed $(i_l,j_l)$ induce shifts of $\pm 1$ on columns and rows $i_l$ and $j_l$ of the similarity matrix $S^\mathrm{match}$, while $S^\mathrm{match}_{i_l,j_l}$ values remain the same. Since there are several corrupted comparisons, we also need to check the values of $\hat S$ at the intersections of rows and columns with indices of corrupted comparisons. Formally, for any $(i,j) \in \{(i_1,j_1), \dots\,(i_{m},j_{m})\}$ and $t \not \in  \{i_1,\dots,i_{m},j_1,\dots,j_{m}\}$
\[ \hat S_{i,t} =
 \left\{\BA{ll}
S_{i,t} + 1  & \mbox{if} \ t  < j \\
S_{i,t} - 1 & \mbox{if} \ t > j,\\
\EA\right. \]
Similarly for  $t \not \in  \{i_1,\dots,i_{m},j_1,\dots,j_{m}\}$
\[ \hat S_{j,t} =
 \left\{\BA{ll}
S_{j,t} - 1  & \mbox{if} \ t  <  i \\
S_{j,t} + 1 & \mbox{if} \ t > i.\\
\EA\right. \]
Let $(s,s')$ and $(t,t') \in \{(i_1,j_1), \dots\,(i_{m},j_{m})\}$, we have
\[
\BA{llcc}
\hat S_{s,t} & =  \frac{1}{2} \left( n+ \sum_{k\neq s',t'} \left( \hat C_{s,k}\hat C_{t,k} \right) + \hat C_{s,s'}\hat C_{t,s'} + \hat C_{s,t'}\hat C_{t,t'} \right) \\
 & =  \frac{1}{2} \left( n+ \sum_{k\neq s',t'} \left( C_{s,k}C_{t,k} \right) - C_{s,s'}C_{t,s'} - C_{s,t'}C_{t,t'} \right)
\EA
\]
Without loss of generality we suppose $s<t$, and since $s<s'$ and $t<t'$, we obtain
 \[
\hat S_{s,t} = 
\left\{\BA{ll}
S_{s,t}   & \mbox{if} \ t>s' \\
S_{s,t} + 2  & \mbox{if} \ t<s'. \\
\EA\right. 
\]
Similar results apply for other intersections of rows and columns with indices of corrupted comparisons (i.e., shifts of $0$, $+2$, or $-2$). For all other coefficients $(s,t)$ such that $s,t \not \in  \{i_1,\dots,i_{m},j_1,\dots,j_{m}\}$, we have $\hat S_{s,t}=S_{s,t}$. We first observe that these last equations, together with our assumption that $j_l-i_l>2$, mean that
\[
\hat S_{s,t} \leq \hat S_{s+1,t} \quad \mbox{and} \quad \hat S_{s,t+1} \leq \hat S_{s,t},\quad \mbox{for any } s < t
\]
so $\hat S$ remains an R-matrix. Moreover, since $j_l-i_l>2$ all these $R$ constraints are strict except between elements of rows $i_l$ and $i_l+1$, and rows $j_l-1$ and $j_l$ (similar for columns). These ties can be broken using the fact that for $k=j_l-1$
\[
\hat S_{i_l,k}=S_{i_l,k}-1<S_{i_l+1,k}-1=\hat S_{i_l+1,k}-1<\hat S_{i_l+1,k}
\] 
which means that $\hat S$ is still a strict R-matrix since $k=j_l -1 >i_l+1$. Moreover, using the same argument as in the proof of Proposition~\ref{prop:sup-oneError}, corrupted comparisons induces $2 m$ ties in the point score vector $w$. 
\end{proof}

For the case of one corrupted comparison, note that the separation condition on the pair of items $(i,j)$ is necessary. When the comparison $C_{i,j}$ between two adjacent items is corrupted, no ranking method can break the resulting tie. For the case of arbitrary number of corrupted comparisons, condition~\eqref{cond:app-distrkerrors} is a sufficient condition only. We study exact ranking recovery conditions with missing comparisons in the Appendix, using similar arguments. We now estimate the number of randomly corrupted entries that can be tolerated while maintaining exact recovery of the true ranking.

\begin{proposition} \label{prop:proba}
Given a comparison matrix for a set of $n$ items with $m$ corrupted comparisons selected uniformly at random from the set of all possible item pairs. Algorithm~\ref{alg:SerialRank} guarantees that the probability of recovery $p(n,m)$ satisfies $p(n,m) \geq 1 - \delta$, provided that $m = O(\sqrt{\delta n})$. In particular, this implies that $p(n,m) = 1 - o(1)$ provided that $m = o(\sqrt{n})$.   
\end{proposition}
\begin{proof} Let $\mathcal{P}$ be the set of all distinct pairs of items from the set $\{1,2,\ldots,n\}$. Let $\mathcal{X}$ be the set of all admissible sets of pairs of items, i.e. containing each $X\subseteq \mathcal{P}$ such that $X$ satisfies condition (\ref{cond:app-distrkerrors}). We consider the case of $m\geq 1$ distinct pairs of items sampled from the set $\mathcal{P}$ uniformly at random without replacement. Let $X_i$ denote the set of sampled pairs given that $i$ pairs are sampled. We seek to bound $p(n,m) = \Prob(X_m \in \mathcal{X})$.
 Given a set of pairs $X\in \mathcal{X}$, let $T(X)$ be the set of non-admissible pairs, i.e. containing $(i,j) \in \mathcal{P}\setminus X$ such that $X \cup (i,j) \notin \mathcal{X}$. We have
\begin{equation}
\Prob(X_m \in \mathcal{X}) = \sum_{x\in \mathcal{X}: |x| = m-1} \left(1-\frac{|T(x)|}{|\mathcal{P}| - (m-1)}\right) \Prob(X_{m-1} = x).
\label{equ:pxm}
\end{equation}
Note that every selected pair from $\mathcal{P}$ contributes at most $15n$ non-admissible pairs. Indeed, given a selected pair $(i,j)$, a non-admissible pair $(s,t)$ should respect one of the following conditions $|s-i|\leq 2$, $|s-j|\leq 2$, $|t-i|\leq 2$, $|t-j|\leq 2$ or $|s-t|\leq2$. Given any item $s$, there are 15 possible choice of $t$ to output a non-admissible pair $(s,t)$, resulting in at most $15n$ non-admissible pairs for the selected pair $(i,j)$.

Hence, for every $x \in \mathcal{X}$ we have
$$
|T(x)| \leq 15n|x|.
$$
Combining this with (\ref{equ:pxm}) and the fact that $|\mathcal{P}| = \binom{n}{2}$, we have
$$
\Prob(X_m \in \mathcal{X}) \geq \left(1-\frac{15n}{\binom{n}{2} - (m-1)}(m-1)\right)\Prob(X_{m-1} \in \mathcal{X}).
$$
 From this it follows
\BEAS
p(n,m) &\geq& \prod_{i=1}^{m-1}\left(1-\frac{15n}{\binom{n}{2} - (i-1)}i\right)\\
&\geq& \prod_{i=1}^{m-1}\left(1-\frac{i}{a(n,m)}\right)
\EEAS
where
$$
a(n,m)= \frac{\binom{n}{2} - (m-1)}{15n}.
$$
Notice that when $m = o(n)$ we have $\left(1-\frac{i}{a(n,m)}\right) \sim \exp (-30i/n)$ and 
$$
\prod_{i=1}^{m-1}\left(1-\frac{i}{a(n,m)}\right) 
\sim \prod_{i=1}^{m-1} \exp (-30i/n)
\sim \exp\left(-\frac{15m^2}{n}\right) \hbox{ for large } n.
$$
Hence, given $\delta > 0$, $p(n,m) \geq 1-\delta$ provided that $m = O(\sqrt{n\delta})$. If $\delta = o(1)$, the condition is $m = o(\sqrt{n})$.
\end{proof}

\section{Spectral Perturbation Analysis}
\label{s:spectralPertAnalysis}
In this section we analyze how \ref{alg:SerialRank} performs when only a small fraction of pairwise comparisons are given. We show that for Erd\"os-R\'enyi graphs, i.e., when pairwise comparisons are observed independently with a given probability, $\Omega(n \log^4 n)$ comparisons suffice for $\ell_2$ consistency of the Fiedler vector and hence $\ell_2$ consistency of the retreived ranking w.h.p. On the other hand we need $\Omega(n^{3/2} \log^4 n)$ comparisons to retrieve a ranking whose local perturbations are bounded in $\ell_\infty$ norm.
Since Erd\"os-R\'enyi graphs are connected with high probability only when the total number of pairs sampled scales as $\Omega(n \log n)$, we need at least that many comparisons in order to retrieve a ranking, therefore the $\ell_2$ consistency result can be seen as optimal up to a polylogarithmic factor. 

Our bounds are mostly related to the work of \citep{Waut13}. In its simplified version \citep[Theorem 4.2][]{Waut13} shows that when ranking items according to their point score, for any precision parameter $\mu \in (0,1)$, sampling independently with fixed probability $\Omega \left(\frac{n\log n}{\mu^2} \right)$ comparisons guarantees that the maximum displacement between the retrieved ranking and the true ranking, i.e., the $\ell_\infty$ distance to the true ranking, is bounded by  $\mu n$ with high probability for $n$ large enough. 

Sample complexity bounds have also been studied for the Rank Centrality algorithm \citep{Dwor01, Nega12}. In their analysis, \citep{ Nega12} suppose that some pairs are sampled independently with fixed probability, and then $k$ comparisons are generated for each sampled pair, under a Bradley-Terry-Luce model (BTL). When ranking items according to the stationary distribution of a transition matrix estimated from comparisons, sampling $\Omega(n \cdot \mbox{polylog}(n))$ 
pairs are enough to bound the relative $\ell_2$ norm perturbation of the stationary distribution. 
However, as pointed out by~\citep{Waut13}, repeated measurements are not practical, e.g., if comparisons are derived from the outcomes of sports games or the purchasing behavior of a customer (a customer typically wants to purchase a product only once). Moreover, \citep{ Nega12} do not provide bounds on the relative $\ell_{\infty}$ norm perturbation of the ranking.

We also refer the reader to the recent work of \citet{Rajk14}, who provide a survey of sample complexity bounds for Rank Centrality, maximum likelihood estimation, least-square ranking and an SVM based ranking, under a more flexible sampling model.
However, those bounds only give the sampling complexity for exact recovery of ranking, which is usually prohibitive when $n$ is large, and are more difficult to interpret.

Finally, we refer the interested reader to \citep{Huan08,shamir2011spectral} for sampling complexity bounds in the context of spectral clustering.

\textbf{Limitations.}
We emphasize that sampling models based on Erd\"os-R\'enyi graphs are not the most realistic, though they have been studied widely in the literature \citep[see for instance][]{feige1994computing,braverman2008noisy, Waut13}. Indeed, pairs are not likely to be sampled independently. For instance, when ranking movies, popular movies in the top ranks are more likely to be compared. Corrupted comparisons are also more likely between items that have close rankings. We hope to extend our perturbation analysis to more general models in future work.

A second limitation of our perturbation analysis comes from the setting of ordinal comparisons, i.e., binary comparisons, since in many applications, several comparisons are provided for each sampled pair. 
Nevertheless, the setting of ordinal comparisons is interesting for the analysis of \ref{alg:SerialRank}, since numerical experiments suggest that it is the setting for which \ref{alg:SerialRank} provides the best results compared to other methods. Note that in practice, we can easily get rid of this limitation (see Section~\ref{s:gln} and \ref{s:numres}).
We refer the reader to numerical experiments in Section~\ref{s:numres}, as well as a recent paper by \citet{cucuringu2015sync}, which introduces another ranking algorithm called SyncRank, and provides extensive numerical experiments on state-of-the-art ranking algorithms, including \ref{alg:SerialRank}.

\textbf{Choice of Laplacian: normalized vs. unnormalized.}
In the spectral clustering literature, several constructions for the Laplacian operators are suggested, namely the unnormalized Laplacian (used in \ref{alg:SerialRank}), the symmetric normalized Laplacian, and the non-symmetric normalized Laplacian. \citet{Von-08} show stronger consistency results for spectral clustering by using the non-symmetric normalized Laplacian. 
Here, we show that the Fiedler vector of the normalized Laplacian is an affine function of the ranking, hence sorting the Fiedler vector still guarantees exact recovery of the ranking, when all comparisons are observed and consistent with a global ranking. In contrast, we only get an asymptotic expression for the unnormalized Laplacian (\cf~section~\ref{s:appendix}). This motivated us to provide an analysis of \ref{alg:SerialRank} robustness based on the normalized Laplacian, though in practice the use of the unnormalized Laplacian is valid and seems to give better results (\cf~Figures~\ref{fig:impactNbObsNoise} and \ref{fig:impactNbObsNoiseNormalized}).

\textbf{Notations.}
Throughout this section, we only focus on the similarity $S^{\mathrm{match}}$ in \eqref{eq:rnk-sim} and write it $S$ to simplify notations. W.l.o.g. we assume in the following that the true ranking is the identity, hence $S$ is an R-matrix.
We write $\|\cdot\|_2$ the operator norm of a matrix, which corresponds to the maximal absolute eigenvalue for symmetric matrices. $\|\cdot\|_F$ denotes the Frobenius norm. We refer to the eigenvalues of the Laplacian as $\lambda_i$, with $\lambda_1=0\leq \lambda_2 \leq \ldots \leq \lambda_n$. For any quantity $x$, we denote by $\tilde x$ its perturbed analogue. We define the residual matrix $R=\tilde S - S$ and write $f$ the normalized Fiedler vector of the Laplacian matrix $L_S$. We define the degree matrix $D_S=\diag(D\ones)$ the diagonal matrix whose elements are the row-sums of matrix $S$. Whenever we use the abreviation w.h.p., this means that the inequality is true with probability greater than $1-2/n$. Finally, we will use $c>0$ for absolute constants, whose values are allowed to vary from one equation to another.

We assume that our information on preferences is both incomplete and corrupted. Specifically, pairwise comparisons are independently sampled with probability $q$ and these sampled comparisons are consistent with the underlying total ranking with probability $p$. Let us define $\tilde C=  B \circ C$ the matrix of observed comparisons, where $C$ is the true comparison matrix defined in~\eqref{eq:c-mat}, $\circ$ is the Hadamard product and $B$ is a symmetric matrix with entries
\[
B_{i,j} = 
\left\{\BA{rl}
0 & \mbox{with probability } 1-q\\
1 & \mbox{with probability } qp\\
-1 & \mbox{with probability } q(1-p).
\EA\right.
\]
In order to obtain an unbiased estimator of the comparison matrix defined in \eqref{eq:c-mat}, we normalize $\tilde C$ by its mean value $q(2p-1)$ and redefine $\tilde S$ as
\[
\tilde S=\frac{1}{q^2(2p-1)^2}\tilde C \tilde C^T + n \ones \ones^T.
\]
For ease of notations we have dropped the factor $1/2$ in \eqref{eq:rnk-sim} w.l.o.g. (positive multiplicative factors of the Laplacian do not affect its eigenvectors).

\subsection{Results}

We now state our main results. The first one bounds $\ell_2$ perturbations of the Fiedler vector $f$ with both missing and corrupted comparisons. Note that $f$ and $\tilde f$ are normalized.
{
\renewcommand{\thetheorem}{\ref{thm:boundFiedlerPerturbation}}
\begin{theorem}
For every $\mu \in (0,1)$ and $n$ large enough, if $q>\frac{\log^4 n}{\mu^2(2p-1)^4 n}$, then 
\[
\|\tilde f -f\|_2 \leq c \frac{\mu}{\sqrt{\log n}}
\]
with probability at least $1-2/n$, where $c>0$ is an absolute constant.
\end{theorem}
\addtocounter{theorem}{-1}
} 

As $n$ goes to infinity the perturbation of the Fiedler vector goes to zero, and we can retrieve the ``true" ranking by reordering the Fiedler vector. Hence this bounds provides  $\ell_2$ consistency of the ranking, with an optimal sampling complexity (up to a polylogarithmic factor). 

The second result bounds local perturbations of the ranking with $\pi$ referring to the ``true" ranking and $\tilde \pi$ to the ranking retrieved by \ref{alg:SerialRank}. 

{
\renewcommand{\thetheorem}{\ref{thm:local}}
\begin{theorem}
For every $\mu \in (0,1)$ and $n$ large enough, if $q>\frac{\log^4 n}{\mu^2(2p-1)^4 \sqrt{n}}$, then 
\[
\|\tilde \pi -\pi\|_{\infty} \leq c  \mu  n
\]
with probability at least $1-2/n$, where $c>0$ is an absolute constant.
\end{theorem}
\addtocounter{theorem}{-1}
} 

This bound quantifies the maximum displacement of any item's ranking. $\mu$ can be seen a ``precision" parameter. For instance, if we set $\mu=0.1$, Theorem~\ref{thm:local} means that we can expect the maximum displacement of any item's ranking to be less than $0.1 \cdot n$ when observing $c^2\cdot 100 \cdot n\sqrt{n} \cdot \log^4 n$ comparisons (with $p=1$).

We conjecture Theorem~\ref{thm:local} still holds true if the condition $q>{\log^4 n}/{\mu^2(2p-1)^4 \sqrt{n}}$ is replaced by the weaker condition $q>{\log^4 n}/{\mu^2(2p-1)^4 n}$.

\subsection{Sketch of the proof.}
The proof of these results relies on classical perturbation arguments and is structured as follows. 
\begin{itemize}
\item {\bf Step 1:} Bound $\|\tilde D_S-D_S\|_2$, $\|\tilde S-S\|_2$ with high probability using concentration inequalities on quadratic forms of Bernoulli variables and results from \citep{Achl07}.
\item {\bf Step 2.} Show that the normalized Laplacian $L=\idm - D^{-1}S$ has a linear Fiedler vector and bound the eigengap between the Fiedler value and other eigenvalues.
\item {\bf Step 3.} Bound $\|\tilde f - f\|_2$ using Davis-Kahan theorem and bounds of steps 1 and 2. 
\item {\bf Step 4.} Use the linearity of the Fiedler vector to translate this result into a bound on the maximum displacement of the retrieved ranking $\|\tilde \pi - \pi\|_{\infty}$.
\end{itemize}

We now turn to the proof itself.

\subsection{{\bf Step 1:} Bounding $\|\tilde D_S-D_S\|_2$ and $\|\tilde S-S\|_2$}

Here, we seek to bound $\|\tilde D_S-D_S\|_2$ and $\|\tilde S-S\|_2$ with high probability using concentration inequalities.

\subsubsection{Bounding the norm of the degree matrix}
We first bound perturbations of the degree matrix with both missing and corrupted comparisons.
\begin{lemma}
\label{lemma:boundDegree}
For every $\mu \in (0,1)$ and $n\geq 100$, if $q \geq \frac{\log^4 n}{\mu^2(2p-1)^4 n}$ then
\[
\|\tilde D_S-D_S\|_2 \leq \frac{3 \mu n^2}{\sqrt{\log n}}
\]
with probability at least $1-1/n$.
\end{lemma}
\begin{proof}
Let $R=\tilde S-S$ and $\delta=\diag D_R= \diag((\tilde S-S) \ones )$. Since $D_S$ and $\tilde D_S$ are diagonal matrices, $\|\tilde D_S-D_S\|_2=\max |\delta_i|$. We first seek a concentration inequality for each $\delta_i$ and then derive a bound on $\| \tilde D_S-D_S\|_2$.

By definition of the similarity matrix $S$ and its perturbed analogue $\tilde S$ we have
\[
R_{ij}=\sum_{k=1}^n C_{ik}C_{jk} \left(\frac{B_{ik}B_{jk}}{q^2(2p-1)^2}-1\right).
\]
Hence
\[
\delta_i=\sum_{j=1}^n R_{ij}= \sum_{j=1}^n \sum_{k=1}^n C_{ik}C_{jk} \left(\frac{B_{ik}B_{jk}}{q^2(2p-1)^2}-1 \right).
\]
Notice that we can arbitrarily fix the diagonal values of $R$ to zeros. Indeed, the similarity between an element and itself should be a constant by convention, which leads to $R_{ii}=\tilde S_{ii}- S_{ii}=0$ for all items $i$. Hence we could take $j\neq i $ in the definition of $\delta_i$, and we can consider $B_{ik}$ independent of $B_{jk}$ in the associated summation.

We first seek a concentration inequality for each $\delta_i$.
Notice that
\begin{eqnarray*}
\delta_i &=& \sum_{j=1}^n \sum_{k=1}^n C_{ik}C_{jk} \left(\frac{B_{ik}B_{jk}}{q^2(2p-1)^2}-1\right) \\
&=& \underbrace{\sum_{k=1}^n \left( \frac{C_{ik}B_{ik}}{q(2p-1)} \sum_{j=1}^n C_{jk} \left(\frac{B_{jk}}{q(2p-1)}-1\right) \right)}_{\mbox{Quad}}
+  \underbrace{\sum_{k=1}^n \sum_{j=1}^n C_{ik}C_{jk} \left(\frac{B_{ik}}{q(2p-1)}-1 \right)}_{\mbox{Lin}}. 
\end{eqnarray*}

The first term (denoted $\mbox{Quad}$ in the following) is quadratic with respect to the $B_{jk}$ while the second term (denoted $\mbox{Lin}$ in the following) is linear. Both terms have mean zero since the $B_{ik}$ are independent of  the $B_{jk}$.
We begin by bounding the quadratic term $\mbox{Quad}$.
Let $X_{jk}= C_{jk} \left(\frac{1}{q(2p-1)}B_{jk}-1\right)$.
We have 
$$\Expect(X_{jk})=C_{jk} \left(\frac{qp-q(1-p)}{q(2p-1)}-1\right)=0,$$
$$\var(X_{jk})=\frac{\var(B_{jk})}{q^2(2p-1)^2}=\frac{1}{q^2(2p-1)^2}(q-q^2(2p-1)^2)=\frac{1}{q(2p-1)^2}-1\leq  \frac{1}{q(2p-1)^2},$$
and
$$|X_{jk}|=\left|\frac{B_{jk}}{q(2p-1)}-1\right| \leq 1+ \frac{1}{q(2p-1)} \leq  \frac{2}{q(2p-1)}\leq  \frac{2}{q(2p-1)^2}.$$
By applying Bernstein's inequality for any $t>0$
\BEQ
\label{eq:Bernstein1}
\Prob \left( \left|\sum_{j=1}^n X_{jk} \right|>t \right)\leq 2 \exp \left( \frac{- q(2p-1)^2t^2}{2(n+2t/3)} \right)\leq 2 \exp \left( \frac{- q(2p-1)^2t^2}{2(n+t)} \right).
\EEQ
Now notice that
\begin{eqnarray*}
\Prob(|\mbox{Quad}|>t) &=&
\Prob \left(\left|\sum_{k=1}^n \left( C_{ik}\frac{B_{ik}}{q(2p-1)} \sum_{j=1}^n X_{jk}  \right)\right| > t \right) \\
&\leq& \Prob\left(\sum_{k=1}^n  \left( \frac{|B_{ik}|}{q(2p-1)}\right) \max_l |\sum_{j=1}^n X_{jl}| >t \right).
\end{eqnarray*}
By applying a union bound to the first Bernstein inequality \eqref{eq:Bernstein1}, for any $t>0$
\[
\Prob \left(\max_l \left|\sum_{j=1}^n X_{jl} \right|>\sqrt{t} \right)\leq 2n \exp \left( \frac{-tq(2p-1)^2}{2(n+\sqrt{t})} \right).
\]
Moreover, since $\Expect |B_{ik}| = q$ we also get from Bernstein's inequality that for any $t>0$
\[
\Prob \left(\sum_{k=1}^n \frac{|B_{ik}|}{q(2p-1)}> \frac{n}{2p-1} + \sqrt{t} \right)\leq  \exp \left( \frac{-tq(2p-1)^2}{2(n+\sqrt{t})} \right).
\]
We deduce from these last three inequalities that for any $t>0$
\[
\Prob(|\mbox{Quad}|>t)
\leq (2n+1) \exp \left( \frac{-tq(2p-1)^2}{2(n+\sqrt{t})} \right).
\]
Taking $t=\mu^2(2p-1)^2 n^2/\log n$ and $q \geq \frac{\log^4 n}{\mu^2(2p-1)^4n}$, with $\mu \leq 1$, we have $\sqrt{t} \leq n$ and we deduce that
\BEQ
\label{eq:quadBoundDegree}
\Prob \left(\left|\mbox{Quad}\right| > \frac{2\mu n^2}{\sqrt{\log n}} \right)
\leq (2n+1) \exp \left( -\frac{\log^3 n}{4}  \right).
\EEQ
We now bound the linear term $\mbox{Lin}$.
\BEAS
\Prob (|\mbox{Lin}|>t)& = &
\Prob \left(\left|\sum_{j=1}^n \sum_{k=1}^n C_{ik}C_{jk} \left(\frac{B_{ik}}{q(2p-1)}-1\right)\right| > t \right) \\
& \leq &\Prob \left(\sum_{k=1}^n |C_{ik}| \max_l |\sum_{j=1}^n X_{jl}| >t \right)\\
& \leq & \Prob \left( \max_k |\sum_{j=1}^n X_{jk}| >t/n \right),
\EEAS
hence
\[
\Prob (|\mbox{Lin}|>t)
\leq 2n \exp \left( \frac{-t^2q(2p-1)^2}{2n^2(n+t/n)} \right).
\]
Taking $t=\mu n^2/(\log n)^{1/2}$ and $q \geq \frac{\log^4 n}{\mu^2(2p-1)^4n}$, with $\mu \leq 1$, we have $t \leq n^2$ and we deduce that
\BEQ
\label{eq:linBoundDegree}
\Prob (|\mbox{Lin}|>t)
\leq 2n \exp \left( -\frac{\log^3 n}{4} \right).
\EEQ
Finally, combining equations \eqref{eq:quadBoundDegree} and \eqref{eq:linBoundDegree}, we obtain for $q \geq \frac{\log^4 n}{\mu^2(2p-1)^4n}$, with $\mu \leq 1$
\[
\Prob \left(|\delta_i| > \frac{3 \mu n^2}{\sqrt{\log n}}\right) \leq (4n+1) \exp \left( -\frac{\log^3 n}{4}  \right).
\]
Now, using a union bound, this shows that for $q \geq \frac{\log^4 n}{\mu^2(2p-1)^4n}$, 
\[
\Prob \left(\max |\delta_i| > \frac{3 \mu n^2}{\sqrt{\log n}} \right) \leq n(4n+1) \exp \left( -\frac{\log^3 n}{4}  \right),
\]
which is less than $1/n$ for $n\geq 100$. 
\end{proof}

\subsubsection{Bounding perturbations of the comparison matrix $C$}

Here, we adapt results in \citep{Achl07} to bound perturbations of the comparison matrix. We will then use bounds on the perturbations of $C$ to bound $\| \tilde S - S\|_2$.

\begin{lemma}
\label{lemma:boundCompMat}
For $n\geq 104$ and $q\geq \frac{\log^3 n}{n}$,
\BEQ
\|C-\tilde C\|_2 \leq  \frac{c}{2p-1} \sqrt{\frac{n}{q}},
\EEQ
with probability at least $1-2/n$, where $c$ is an absolute constant.
\end{lemma}
\begin{proof}
The main argument of the proof is to use the independence of the $C_{ij}$ for $i<j$ in order to bound $\|\tilde C - C\|_2$ by a constant times $\sigma \sqrt{n}$, where $\sigma$ is the standard deviation of $C_{ij}$. To isolate independent entries in the perturbation matrix, we first need to break the anti-symmetry of $\tilde C-C$ by decomposing $X=\tilde C-C$ into its upper triangular part and its lower triangular part, i.e.,~$\tilde C-C=X_{\mathrm{up}}+X_{\mathrm{low}}$, with $X_{\mathrm{up}}=-X_{\mathrm{low}}^T$ (diagonal entries of $\tilde C-C$ can be arbitrarily set to 0).
Entries of $X_{\mathrm{up}}$ are all independent, with variance less than the variance of $\tilde C_{ij}$. Indeed, lower entries of  $X_{\mathrm{up}}$ are equal to 0 and hence have variance 0. Notice that 
\[
\|\tilde C-C\|_2=\|X_{\mathrm{up}} + X_{\mathrm{low}}\|_2 \leq \|X_{\mathrm{up}}\|_2 +\|X_{\mathrm{low}}\|_2 \leq 2 \|X_{\mathrm{up}}\|_2,
\]
so bounding  $\|X_{\mathrm{up}}\|_2$ will give us a bound on  $\|X\|_2$. In the rest of the proof we write $X_{\mathrm{up}}$ instead of $X$ to simplify notations. We can now apply \citep[Th.\,3.1]{Achl07} to $X$. Since $$X_{ij} = \tilde C_{ij}-C_{ij} = C_{ij} \left( \frac{B_{ij}}{q(2p-1)} -1 \right),$$ we have (\cf~proof of Lemma~\ref{lemma:boundDegree}) $\Expect(X_{ij})=0$,
$\var(X_{ij})\leq \frac{1}{q(2p-1)^2}$, and $|X_{ij}|\leq \frac{2}{q(2p-1)}$.
Hence for a given $\epsilon >0$ such that 
\begin{equation}
\label{eq:condBoundCompMat}
\frac{4}{q(2p-1)} \leq \left( \frac{\log(1+\epsilon)}{2\log(2n)} \right)^2 ~  \frac{\sqrt{2n}}{\sqrt{q}(2p-1)} ,
\end{equation}
for any $\theta>0$ and $n\geq76$,
\begin{equation}
\label{eq:boundCompMat1}
\Prob\left(\|X\|_2 \geq  2(1+\epsilon + \theta) \frac{1}{\sqrt{q}(2p-1)} \sqrt{2n}\right)<2\exp \left(-16 \frac{\theta^2}{\epsilon^4}\log^3 n \right).
\end{equation}
For $q\geq \frac{(\log 2n)^3}{n}$ and taking $\epsilon \geq \exp(\sqrt(16/\sqrt(2)))- 1$ (so $\log(1+\epsilon)^2\geq 16/\sqrt{2}$) means inequality~\eqref{eq:condBoundCompMat} holds. Taking~\eqref{eq:boundCompMat1} with $\epsilon=30$ and $\theta=30$ we get
\begin{equation}
\label{eq:boundCompMat2}
\Prob\left(\|X\|_2 \geq \frac{2\sqrt{2}(1+30+30)}{2p-1} \sqrt{\frac{n}{q}}\right)<2\exp \left(-10^{-2}\log^3 n \right).
\end{equation}
Hence for $n\geq104$, we have $\log^3 n>100$ and 
\[
\Prob\left(\|X\|_2 \geq \frac{173}{2p-1} \sqrt{\frac{n}{q}}\right)<2/n.
\]
Noting that $\log 2n \leq 1.15 \log n$ for $n \geq 104$, we obtain the desired result by choosing $c=2\times173\times\sqrt{1.15}\leq371$.
\end{proof}

\subsubsection{Bounding the perturbation of the similarity matrix $\|S\|$.}

We now seek to bound $\| \tilde S - S\|$ with high probability.

\begin{lemma}
\label{lemma:boundSimilarityPert}
For every $\mu \in (0,1)$, $n\geq 104$, if $q>\frac{\log^4 n}{\mu^2(2p-1)^2 n}$, then 
\[
\|\tilde S -S\|_2 \leq c \frac{\mu n^2}{\sqrt{\log n}},
\]
with probability at least $1-2/n$, where $c$ is an absolute constant.
\end{lemma}
\begin{proof}
Let $X=\tilde C-C$.
We have
\[
\tilde C \tilde C^T= (C+X)(C+X)^T=CC^T + XX^T + XC^T + CX^T,
\]
hence
\[
\tilde S - S= XX^T + XC^T + CX^T ,
\]
and
\[
\|\tilde S - S\|_2 \leq \|XX^T\|_2 + \|XC^T\|_2 + \|CX^T\|_2 
\leq \|X\|_2^2 + 2\|X\|_2  \|C\|_2.
\]
From Lemma~\ref{lemma:boundCompMat} we deduce that for $n\geq 104$ and $q\geq \frac{\log^4 n}{n}$, with probability at least $1-2/n$
\BEQ
\|\tilde S - S\|_2 \leq   \frac{c^2n}{q(2p-1)^2} + \frac{2 c}{2p-1} \sqrt{\frac{n}{q}} \|C\|_2 .
\EEQ
Notice that $\|C\|_2^2\leq \Tr (CC^T) = n^2$, hence $\|C\|_2\leq n$ and 
\BEQ
\label{eq:boundSDiff}
\|\tilde S - S\|_2 \leq   \frac{c^2n}{q(2p-1)^2} + \frac{2 c n}{2p-1} \sqrt{\frac{n}{q}} .
\EEQ
By taking $q>\frac{\log^4 n}{\mu^2(2p-1)^2 n}$, we get for $n\geq 104$ with probability at least $1-2/n$
\[
\|\tilde S - S\|_2 \leq   \frac{c^2 \mu^2 n^2}{\log^4 n} + \frac{2 c \mu n^2}{\log^2 n}.
\]
Hence setting a new constant $c$ with $c= \max (c^2 (\log 104)^{-7/2}, 2c(\log 104)^{-3/2}) \leq 270 $,
\[
\|\tilde S -S\|_2 \leq c \frac{\mu n^2}{\sqrt{\log n}}
\]
with probability at least $1-2/n$, which is the desired result.
\end{proof}

\subsection{Step 2: Controlling the eigengap}

In the following proposition we show that the normalized Laplacian of the similarity matrix $S$ has a constant Fiedler value and a linear Fiedler vector. We then deduce bounds on the eigengap between the first, second and third smallest eigenvalues of the Laplacian.

\begin{proposition}
\label{prop:FiedlerVecExpr}
Let $L^\mathrm{norm}=\idm-D^{-1}S$ be the non-symmetric normalized Laplacian of $S$. $L^\mathrm{norm}$ has a linear Fiedler vector, and its Fiedler value is equal to $2/3$.
\end{proposition}

\begin{proof}
Let $x_i=i-\frac{n+1}{2}$ ($x$ is linear with mean zero). We want to show that
$ L^\mathrm{norm}x =\lambda_2 x$ or equivalently $Sx=(1-\lambda_2)Dx.$ We develop both sides of the last equation, and use the following facts
\[
S_{i,j}=n-|j-i|,\quad \sum_{k=1}^n k = \frac{n(n+1)}{2}, \quad \sum_{k=1}^n k^2 = \frac{n(n+1)(2n+1)}{6}.
\]
We first get an expression for the degree of $S$, defined by $d=S\ones=\sum_{i=1}^n S_{i,k}$, with
\begin{eqnarray*}
d_i & = & \sum_{k=1}^{i-1} S_{i,k} + \sum_{k=i}^{n} S_{i,k} \\
& =& \sum_{k=1}^{i-1} (n-i+k) + \sum_{k=i}^{n} (n-k+i) \\
& = & \frac{n(n-1)}{2} + i(n-i+1).
\end{eqnarray*}
Similarly we have
\begin{eqnarray*}
\sum_{k=1}^n kS_{i,k}
& =& \sum_{k=1}^{i-1} k(n-i+k) + \sum_{k=i}^{n} k(n-k+i) \\
& =& \frac{n^2(n+1)}{2} + \frac{i(i-1)(2i-1)}{3}-\frac{n(n+1)(2n+1)}{6}-i^2(i-1)+i\frac{n(n+1)}{2}.
\end{eqnarray*}
Finally, setting $\lambda_2=2/3$, notice that
\begin{eqnarray*}
[Sx]_i
& =& \sum_{k=1}^{n}  S_{i,k}\left(k-\frac{n+1}{2} \right)\\
& =& \sum_{k=1}^{n}  kS_{i,k} - \frac{n+1}{2}d_i \\
& =& \frac{1}{3}\left(\frac{n(n-1)}{2} +i(n-i+1) \right) \left(i-\frac{n+1}{2} \right)\\
& =&(1-\lambda_2) d_i x_i,
\end{eqnarray*}
which shows that $Sx=(1-\lambda_2)Dx$.
\end{proof}

The next corollary will be useful in following proofs.
\begin{corollary}
\label{cor:boundMaxFiedler}
The Fiedler vector $f$ of the unperturbed Laplacian satisfies $\|f\|_\infty \leq 2/\sqrt{n}$. 
\end{corollary}
\begin{proof}
We use the fact that $f$ is collinear to the vector $x$ defined by $x_i=i-\frac{n+1}{2}$ and verifies $\|f\|_2=1$.
Let us consider the case of $n$ odd. The Fiedler vector verifies $f_i = \frac{i - (n+1)/2 }{a_n}$, with 
\[
a_n^2 = 2 \sum_{k=0}^{(n-1)/2} k^2 = \frac{2}{6}\frac{n-1}{2}\left(\frac{n-1}{2}+1\right)((n-1)+1)= \frac{n^3-n}{12}.
\]
Hence
\[
\|f\|_{\infty} = f_n = \frac{n-1}{2a_n} \leq \sqrt{\frac{3}{n-1}} \leq \frac{2}{\sqrt{n}} \ \ \mbox{for} \ \  n \geq 5.
\]
A similar reasoning applies for $n$ even.
\end{proof}

\begin{lemma}\label{lem:eigengap}
The minimum eigengap between the Fiedler value and other eigenvalues is bounded below by a constant for $n$ sufficiently large.
\end{lemma}
\begin{proof}
The first eigenvalue of the Laplacian is always 0, so we have for any $n$, $\lambda_2 - \lambda_1 = \lambda_2=2/3$.
Moreover, using results from \citep{Von-08}, we know that eigenvalues of the normalized Laplacian that are different from one converge to an asymptotic spectrum, and that the limit eigenvalues are ``isolated". Hence there exists $n_0>0$ and $c>0$ such that for any $n\geq n_0$ we have $\lambda_3-\lambda_2>c$.
\end{proof}

Numerical experiments show that $\lambda_3$ converges to $0.93\ldots$ very fast when $n$ grows towards infinity. 

\subsection{Step 3: Bounding the perturbation of the Fiedler vector $\|\tilde f -f\|_2$}
We can now compile results from previous sections to get a first perturbation bound and show $\ell_2$ consistency of the Fiedler vector when comparisons are both missing and corrupted.

\begin{theorem}
\label{thm:boundFiedlerPerturbation}
For every $\mu \in (0,1)$ and $n$ large enough, if $q>\frac{\log^4 n}{\mu^2(2p-1)^4 n}$, then 
\[
\|\tilde f -f\|_2 \leq c \frac{\mu}{\sqrt{\log n}},
\]
with probability at least $1-2/n$.
\end{theorem}
\begin{proof}
In order to use Davis-Kahan theorem, we need to relate perturbations of the normalized Laplacian matrix to perturbations of the similarity and degree matrices. To simplify notations, we write $L=\idm-D^{-1}S$ and $\tilde L=\idm-\tilde D^{-1}\tilde S$. 

Since the normalized Laplacian is not symmetric, we will actually apply Davis-Kahan theorem to the symmetric normalized Laplacian $L_{sym}=\idm - D^{-1/2}S D^{-1/2}$. It is easy to see that $L_{sym}$ and $L$ have the same Fiedler value, and that the Fiedler vector $f_{sym}$ of $L_{sym}$ is equal to $D^{1/2}f$ (up to normalization). Indeed, if $v$ is the eigenvector associated to the $i^{th}$ eigenvalue of $L$ (denoted by $\lambda_i$), then
\[
L_{sym}D^{1/2}v=D^{-1/2}(D-S)D^{-1/2}D^{1/2}v=D^{-1/2}(D-S)v=D^{1/2}(\idm -D^{-1}S)v=\lambda_i D^{1/2} v.
\]
Hence perturbations of the Fiedler vector of $L_{sym}$ are directly related to perturbations of the Fiedler vector of $L$.

The proof relies mainly on Lemma~\ref{lemma:boundDegree}, which states that for $n \geq 100$, denoting by $d$ the vector of diagonal elements of $D_S$,
\[
\|D_R\|_2=\max |\tilde d_i - d_i| \leq \frac{3 \mu n^2}{\sqrt{\log n}}
\]
with probability at least $1-\frac{2}{n}$.
Combined with the fact that $d_i=\frac{n(n-1)}{2} + i(n-i+1)$ (\cf~proof of Proposition~\ref{prop:FiedlerVecExpr}), this guarantees that $d_i$ and $\tilde d_i$ are strictly positive. Hence $D^{-1/2}$ and $\tilde D^{-1/2}$ are well defined.
We now decompose the perturbation of the Laplacian matrix.
Let $\Delta=D^{-1/2}$, we have
\begin{eqnarray*}
\| \tilde L_{sym} - L_{sym} \|_2 & = & \| \tilde \Delta \tilde S \tilde \Delta - \Delta S \Delta \|_2 \\
& = & \| \tilde \Delta \tilde S \tilde \Delta - \tilde \Delta S  \tilde \Delta  +  \tilde \Delta S \tilde \Delta  -  \Delta S \Delta \|_2 \\
& = & \| \tilde \Delta (\tilde S -S) \tilde \Delta  + \tilde \Delta S \tilde \Delta  -  \Delta S \tilde \Delta +\Delta S \tilde \Delta -\Delta S  \Delta \|_2 \\
& = & \| \tilde \Delta (\tilde S -S) \tilde \Delta  + (\tilde \Delta - \Delta)S \tilde \Delta  +  \Delta S (\tilde \Delta - \Delta) \|_2 \\
& \leq & \|\tilde \Delta \|_2^2 \|\tilde S -S\|_2 + \|S\|_2 (\|\tilde \Delta\|_2+ \|\Delta\|_2) \|\tilde \Delta - \Delta \|_2 .
\end{eqnarray*}

We first bound $\|\tilde \Delta - \Delta \|_2$. Notice that
\[
\|\tilde \Delta - \Delta \|_2= \max_i | \tilde  d_i^{-1/2}- d_i^{-1/2}|,
\]
where $d_i$ (respectively $\tilde  d_i$) is the sum of elements of the $i^{\mathrm{th}}$ row of $S$ (respectively $\tilde S$).
Hence 
\[
\|\tilde \Delta - \Delta \|_2= \max_i  \frac{\left| \sqrt{ \tilde d_i} - \sqrt{d_i} \right|}{\sqrt{\tilde d_i d_i}}
= \max_i  \frac{\left|  \tilde d_i -  d_i \right|}{\sqrt{\tilde d_i d_i} (\sqrt{ \tilde d_i} + \sqrt{d_i} )} .
\]
Using Lemma~\ref{lemma:boundDegree} we obtain
\[
\|\tilde \Delta-\Delta\|_2 \leq \max_i \frac{\frac{3 \mu n^2}{\sqrt{\log n}}}{\sqrt{d_i}(d_i-\frac{3 \mu n^2}{\sqrt{\log n}})+ d_i \sqrt{d_i-\frac{3 \mu n^2}{\sqrt{\log n}}}}, \quad i=1,\ldots,n, \mbox{ w.h.p.}
\]
Since $d_i=\frac{n(n-1)}{2} + i(n-i+1)$ (\cf~proof of Proposition~\ref{prop:FiedlerVecExpr}), for $\mu<1$ there exists a constant $c$ such that
$d_i>d_i-\frac{3 \mu n^2}{\sqrt{\log n}}> c n^2$. 
We deduce that there exists an absolute constant $c$ such that
\begin{equation}
\label{eq:boundPerturbDegree}
\|\tilde \Delta-\Delta\|_2 \leq \frac{ c \mu}{n\sqrt{\log n}} \mbox{ w.h.p}.
\end{equation}

Similarly we obtain that 
\begin{equation}
\label{eq:boundDegreeTrue}
\|\Delta \|_2 \leq \frac{ c }{n} \mbox{ w.h.p},
\end{equation}
and
\begin{equation}
\label{eq:boundDegreetilde}
\|\tilde \Delta \|_2 \leq \frac{ c }{n} \mbox{ w.h.p}.
\end{equation}

Moreover, we have 
\[
\|S\|_2=\|CC^T+n\ones \ones^T \|_2 \leq \|C\|^2_2+n \| \ones \ones^T \|_2 \leq 2n^2.
\]
Hence,
\[
\|S\|_2 (\|\tilde \Delta\|_2+ \|\Delta\|_2)\|\tilde \Delta -\Delta\|_2 \leq \frac{ c\mu }{\sqrt{\log n}} \mbox{ w.h.p},
\]
where $c:=4c^2$.
Using Lemma~\ref{lemma:boundSimilarityPert}, we can similarly bound $\|\tilde \Delta\|_2^2 \|\tilde S -S\|_2$ and obtain\begin{equation}
\label{eq:boundLaplacianPert}
\|\tilde L_{sym} - L_{sym}\|_2 \leq \frac{ c \mu}{\sqrt{\log n}} \mbox{ w.h.p},
\end{equation}
where $c$ is an absolute constant.
Finally, for small $\mu$, Weyl's inequality, equation~\eqref{eq:boundLaplacianPert} together with Lemma~\ref{lem:eigengap} ensure that for $n$ large enough with high probability $|\tilde \lambda_3 - \lambda_2|> |\lambda_3-\lambda_2|/2$ and $|\tilde \lambda_1 - \lambda_2|> |\lambda_1-\lambda_2|/2$. Hence we can apply Davis-Kahan theorem. 
Compiling all constants into $c$ we obtain
\begin{equation}
\label{eq:boundperturbfsym}
\|\tilde f_{sym} -f_{sym}\|_2 \leq  \frac{c\mu}{\sqrt{\log n}} \mbox{ w.h.p}.
\end{equation}

Finally we relate the perturbations of $f_{sym}$ to the perturbations of $f$.
Since $f_{sym}=\frac{D^{1/2} f}{\|D^{1/2} f\|_2}$, letting $\alpha_n=\|D^{1/2} f\|$, we deduce that
\begin{eqnarray*}
\| \tilde f -f  \|_2 & = & \| \tilde \alpha_n \tilde \Delta  \tilde  f_{sym} -  \alpha_n \Delta f_{sym} \|_2 \\
& = & \| \Delta  (\tilde \alpha_n \tilde  f_{sym} -\alpha_n f_{sym} ) + \tilde \alpha_n (\tilde \Delta - \Delta) \tilde  f_{sym}  \|_2 \\
& \leq & \| \Delta \|_2 \| \tilde \alpha_n \tilde f_{sym} - \alpha_n f_{sym}\|_2 + \| \tilde \alpha_n \|_2 \|\tilde \Delta - \Delta \|_2 .
\end{eqnarray*}
Similarly as for inequality \eqref{eq:boundPerturbDegree}, we can show that $\|\tilde D^{1/2} \| $ and  $\| D^{1/2} \|$ are of the same order $O(n)$. Since $\|f\|_2=\|\tilde f\|_2=1$, this is also true for $\|\alpha_n\|_2$ and $\|\tilde \alpha_n\|_2$.
We conclude the proof using inequalities \eqref{eq:boundPerturbDegree}, \eqref{eq:boundDegreeTrue} and \eqref{eq:boundperturbfsym}.
\end{proof}

\subsection{Bounding ranking perturbations $\|\tilde \pi -\pi\|_{\infty}$}
\ref{alg:SerialRank}'s ranking is derived by sorting the Fiedler vector. While the consistency result in Theorem~\ref{thm:boundFiedlerPerturbation} shows the $\ell_2$ estimation error going to zero as $n$ goes to infinity, this is not sufficient to quantify the maximum displacement of the ranking. To quantify the maximum displacement of the ranking, as in \citep{Waut13}, we need to bound $\|\tilde \pi -\pi\|_{\infty}$ instead. 

We bound the maximum displacement of the ranking here with an extra factor $\sqrt{n}$ compared to the sampling rate in \citep{Waut13}. We would only need a better component-wise bound on $\tilde S -S$ to get rid of this extra factor $\sqrt{n}$, and we hope to achieve it in future work.

The proof is in two parts: we first bound the $\ell_{\infty}$ norm of the perturbation of the Fiedler vector, then translate this perturbation of the Fiedler vector into a perturbation of the ranking.

\subsubsection{Bounding the $\ell_{\infty}$ norm of the Fiedler vector perturbation}
We start by a technical lemma bounding $\|(\tilde S-S)f\|_{\infty}$.

\begin{restatable}{lemma}{lemmaboundRf}
\label{lemma:boundRf}
Let $r>0$, 
for every $\mu \in (0,1)$ and $n$ large enough, if $q>\frac{\log^4 n}{\mu^2(2p-1)^4 n}$, then 
\[
\|(\tilde S-S)f\|_{\infty} \leq \frac{3 \mu n^{3/2}}{\sqrt{\log n}}
\]
with probability at least $1-2/n$.
\end{restatable}

\begin{proof}
The proof is very much similar to the proof of Lemma~\ref{lemma:boundDegree} and can be found the Appendix (section~\ref{s:appendixPertAnalysis}).
\end{proof}

We now prove the main result of this section, bounding $\|\tilde f -f\|_{\infty}$ with high probability when roughly $O(n^{3/2})$ comparisons are sampled.

\begin{restatable}{lemma}{lemmaboundLinfFiedler}
\label{lemma:boundLinfFiedler}
For every $\mu \in (0,1)$ and $n$ large enough, if $q>\frac{\log^4 n}{\mu^2(2p-1)^4 \sqrt{n}}$, then \[
\|\tilde f -f\|_{\infty} \leq c \frac{\mu}{\sqrt{n\log n}}
\]
with probability at least $1-2/n$, where $c$ is an absolute constant.
\end{restatable}

\begin{proof}
Notice that by definition
$\tilde L \tilde f =\tilde \lambda_2 \tilde f$ and $Lf=\lambda_2 f$. Hence for $\tilde \lambda_2 >0$
\begin{eqnarray*}
\tilde f -f &=& \frac{\tilde L \tilde f}{\tilde \lambda_2} -f \\
&=& \frac{\tilde L \tilde f - Lf}{\tilde \lambda_2} + \frac{(\lambda_2 - \tilde \lambda_2)f}{\tilde \lambda_2}.
\end{eqnarray*}
Moreover
\begin{eqnarray*}
\tilde L \tilde f - Lf & = & (\idm - \tilde D^{-1}\tilde S)\tilde f - (\idm -D^{-1}S)f \\
& = & (\tilde f  -  f) + D^{-1} S f - \tilde D^{-1}\tilde S \tilde f \\
& = & (\tilde f  -  f) + D^{-1} S f - \tilde D^{-1}\tilde S  f + \tilde D^{-1}\tilde S  f - \tilde D^{-1}\tilde S  \tilde f \\
& = & (\tilde f  -  f) + (D^{-1} S - \tilde D^{-1}\tilde S)f + \tilde D^{-1}\tilde S (f - \tilde f) 
\end{eqnarray*}
Hence
\BEQ
(\idm(\tilde \lambda_2-1)+\tilde D^{-1}\tilde S)(\tilde f-f)  =  (D^{-1} S- \tilde D^{-1}\tilde S+ (\lambda_2 - \tilde \lambda_2)\idm)f .
\EEQ
Writing $S_i$ the $i^{th}$ row of $S$ and $d_i$ the degree of row $i$, using the triangle inequality, we deduce that
\BEQ
| \tilde f_i-f_i | \leq  \frac{1}{|\tilde \lambda_2 -1|}\left(|(d_i^{-1} S_i- \tilde d_i^{-1}\tilde S_i) f| + |\lambda_2 - \tilde \lambda_2| |f_i| + |\tilde d_i^{-1}\tilde S_i (\tilde f-f)| \right).
\EEQ

It remains to bound each term separately, using Weyl's inequality for the denominator and previous lemmas for numerator terms, which is detailed in the Appendix (section~\ref{s:appendixPertAnalysis}). 
\end{proof}

\subsubsection{Bounding the $\ell_{\infty}$ norm of the ranking perturbation}
First note that the $\ell_{\infty}$-norm of the ranking perturbation is equal to the number of pairwise disagreements between the true ranking and the retrieved one, i.e.,~for any $i$
\[
|\tilde \pi_i - \pi_i| = \sum_{j<i} \ones_{\tilde f_j> \tilde f_i} + \sum_{j>i} \ones_{\tilde f_j <\tilde f_i}.
\]
Now we will argue that when $i$ and $j$ are far apart, with high probability 
\[
\tilde f_j - \tilde f_i= (\tilde f_j - f_j) + (f_j-f_i) + (f_i - \tilde f_i)
\]
will have the same sign as $j-i$.
Indeed $|\tilde f_j - f_j|$ and $|\tilde f_i - f_i|$ can be bounded with high probability by a quantity less than $|f_j-f_i|/2$ for $i$ and $j$ sufficiently ``far apart".
Hence, $|\tilde \pi_i - \pi_i|$  is bounded by the number of pairs that are not sufficiently ``far apart".
We quantify the term ``far apart" in the following proposition.

\begin{theorem}
\label{thm:local}
For every $\mu \in (0,1)$ and  $n$ large enough, if $q>\frac{\log^4 n}{\mu^2(2p-1)^2 \sqrt{n}}$, then 
\[
\|\tilde \pi -\pi\|_{\infty} \leq c  \mu  n,
\]
with probability at least $1-2/n$, where $c$ is an absolute constant.
\end{theorem}
\begin{proof}
We assume w.l.o.g. in the following that the true ranking is the identity, hence the unperturbed Fiedler vector $f$ is strictly increasing. We first notice that for any $j>i$
\[
\tilde f_j - \tilde f_i= (\tilde f_j - f_j) + (f_j-f_i) + (f_i - \tilde f_i).
\]
Hence for any $j>i$
\[
\|\tilde f - f\|_{\infty} \leq  \frac{|f_j-f_i|}{2} \implies \tilde f_j \geq \tilde f_i.
\]
Consequently, fixing an index $i_0$,
\[
\sum_{j>i_0} \ones_{\tilde f_j <\tilde f_{i_0}} 
\leq \sum_{j>i_0} \ones_{\|\tilde f - f\|_{\infty} >  \frac{|f_j-f_{i_0}|}{2}}.
\]
Now recall that by Lemma~\ref{lemma:boundLinfFiedler}, for $q>\frac{\log^4 n}{\mu^2(2p-1)^2 \sqrt{n}}$
\[
\|\tilde f -f\|_{\infty} \leq c \frac{\mu }{\sqrt{n\log n}}
\]
 with probability at least $1-2/n$.
Hence 
\[
\sum_{j>i_0} \ones_{\tilde f_j <\tilde f_{i_0}} 
\leq \sum_{j>i_0} \ones_{\|\tilde f - f\|_{\infty} >  \frac{|f_j-f_{i_0}|}{2}}
\leq \sum_{j>i_0} \ones_{\frac{c\mu }{\sqrt{n\log n}}>  \frac{|f_j-f_{i_0}|}{2}}  \ \ \ \mbox{w.h.p.}
\]
We now consider the case of $n$ odd (a similar reasoning applies for $n$ even).
We have $f_{j} = \frac{j - (n+1)/2 }{a_n}$ for all $j$, with 
\[
a_n^2 = 2 \sum_{k=0}^{(n-1)/2} k^2 = \frac{2}{6}\frac{n-1}{2}\left(\frac{n-1}{2}+1\right)((n-1)+1)= \frac{n^3-n}{12}.
\]
Therefore
\[
\frac{c \mu }{\sqrt{n\log n}}>  \frac{|f_j-f_{i_0}|}{2}
\iff  \frac{c \mu }{\sqrt{n\log n}} > \frac{|j-i_0| \sqrt{3}}{n^{3/2}}
\iff  \frac{ c\mu n}{\sqrt{3\log n}} > |j-i_0|.
\]
Dividing $c$ by $\sqrt{3}$, we deduce that 
\[
\sum_{j>i_0} \ones_{\tilde f_j <\tilde f_{i_0}} 
\leq \sum_{j>i_0} \ones_{ \frac{c\mu n}{\sqrt{\log n}} > |j-i_0|} 
= \left\lfloor \frac{c \mu n}{\sqrt{\log n}} \right\rfloor
\leq \frac{c \mu n}{\sqrt{\log n}}  \ \ \ \mbox{w.h.p.}
\]
Similarly 
\[
\sum_{j<i_0} \ones_{\tilde f_j >\tilde f_{i_0}} \leq \frac{c \mu n}{\sqrt{\log n}}  \ \ \ \mbox{w.h.p.}
\]
Finally, we obtain
\[
|\tilde \pi_{i_0} - \pi_{i_0}| = \sum_{j<i_0} \ones_{\tilde f_j> \tilde f_{i_0}} + \sum_{j>i_0} \ones_{\tilde f_j <\tilde f_{i_0}}
 \leq \frac{c \mu n}{\sqrt{\log n}}  \ \ \ \mbox{w.h.p.},
\]
where $c$ is an absolute constant.
Since the last inequality relies on $\|\tilde f -f\|_{\infty} \leq  \frac{c\mu }{\sqrt{n\log n}}$, it is true for all $i_0$ with probabilty $1-2/n$, which concludes the proof.
\end{proof}

\section{Numerical Experiments}\label{s:numres}

We now describe numerical experiments using both synthetic and real datasets to compare the performance of \ref{alg:SerialRank} with several classical ranking methods. 

\subsection{Synthetic Datasets} The first synthetic dataset consists of a matrix of pairwise comparisons derived from a given ranking of $n$ items with uniform, randomly distributed corrupted or missing entries. A second synthetic dataset consists of a full matrix of pairwise comparisons derived from a given ranking of $n$ items, with added ``local'' noise on the similarity between nearby items. Specifically, given a positive integer $m$, we let $C_{i,j} = 1$ if $i<j-m$, $C_{i,j} \sim \mathrm{Unif}[-1,1]$ if $|i-j| \leq m$, and $C_{i,j} = -1$ if $i > j+m$. 
In Figure~\ref{fig:impactNbObsNoise}, we measure the Kendall~$\tau$ correlation coefficient between the true ranking and the retrieved ranking, when varying either the percentage of corrupted comparisons or the percentage of missing comparisons. Kendall's~$\tau$ counts the number of agreeing pairs minus the number of disagreeing pairs between two rankings, scaled by the total number of pairs, so that it takes values between -1 and~1. Experiments were performed with $n=100$ and reported Kendall $\tau$ values were averaged over 50 experiments, with standard deviation less than 0.02 for points of interest (i.e., with Kendall $\tau>0.8$). 

Results suggest that \ref{alg:SerialRank} (SR, full red line) produces more accurate rankings than point score (PS, \citep{Waut13} dashed blue line), Rank Centrality (RC \citep{Nega12} dashed green line), and maximum likelihood (BTL \citep{Brad52}, dashed magenta line) in regimes with limited amount of corrupted and missing comparisons. In particular \ref{alg:SerialRank} seems more robust to corrupted comparisons. On the other hand, the performance deteriorates more rapidly in regimes with very high number of corrupted/missing comparisons.
For a more exhaustive comparison of \ref{alg:SerialRank} to state-of-the art ranking algorithms, we refer the interested reader to a recent paper by \citet{cucuringu2015sync}, which introduces another ranking algorithm called SyncRank, and provides extensive numerical experiments.

\subsection{Real Datasets} The first real dataset consists of pairwise comparisons derived from outcomes in the TopCoder algorithm competitions. We collected data from 103 competitions among 2742 coders over a period of about one year. Pairwise comparisons are extracted from the ranking of each competition and then averaged for each pair. TopCoder maintains ratings for each participant, updated in an online scheme after each competition, which were also included in the benchmarks. 
To measure performance in Figure~\ref{fig:TopCoder}, we compute the percentage of upsets (i.e. comparisons disagreeing with the computed ranking), which is closely related to the Kendall $\tau$ (by an affine transformation if comparisons were coming from a consistent ranking). We refine this metric by considering only the participants appearing in the top $k$, for various values of $k$, i.e. computing
\BEQ\label{eq:loss}
l_k=\frac{1}{|\mathcal{C}_k|} \sum_{i,j \in \mathcal{C}_k} \ones_{r(i)>r(j)}\ones_{C_{i,j}<0},
\EEQ
where $\mathcal{C}$ are the pairs $(i,j)$ that are compared and such that $i,j$ are both ranked in the top $k$, and $r(i)$ is the rank of $i$. Up to scaling, this is the loss considered in~\citep{Keny07}.

This experiment shows that \ref{alg:SerialRank} gives competitive results with other ranking algorithms. Notice that  rankings could probably be refined by designing a similarity matrix taking into account the specific nature of the data.

\begin{figure}[t]
\begin{center}
\begin{tabular}{ccc}
\psfrag{Kendall}[b][t]{Kendall $\tau$}
\psfrag{corrupted}[t][b]{\% corrupted}
\includegraphics[scale=0.33]{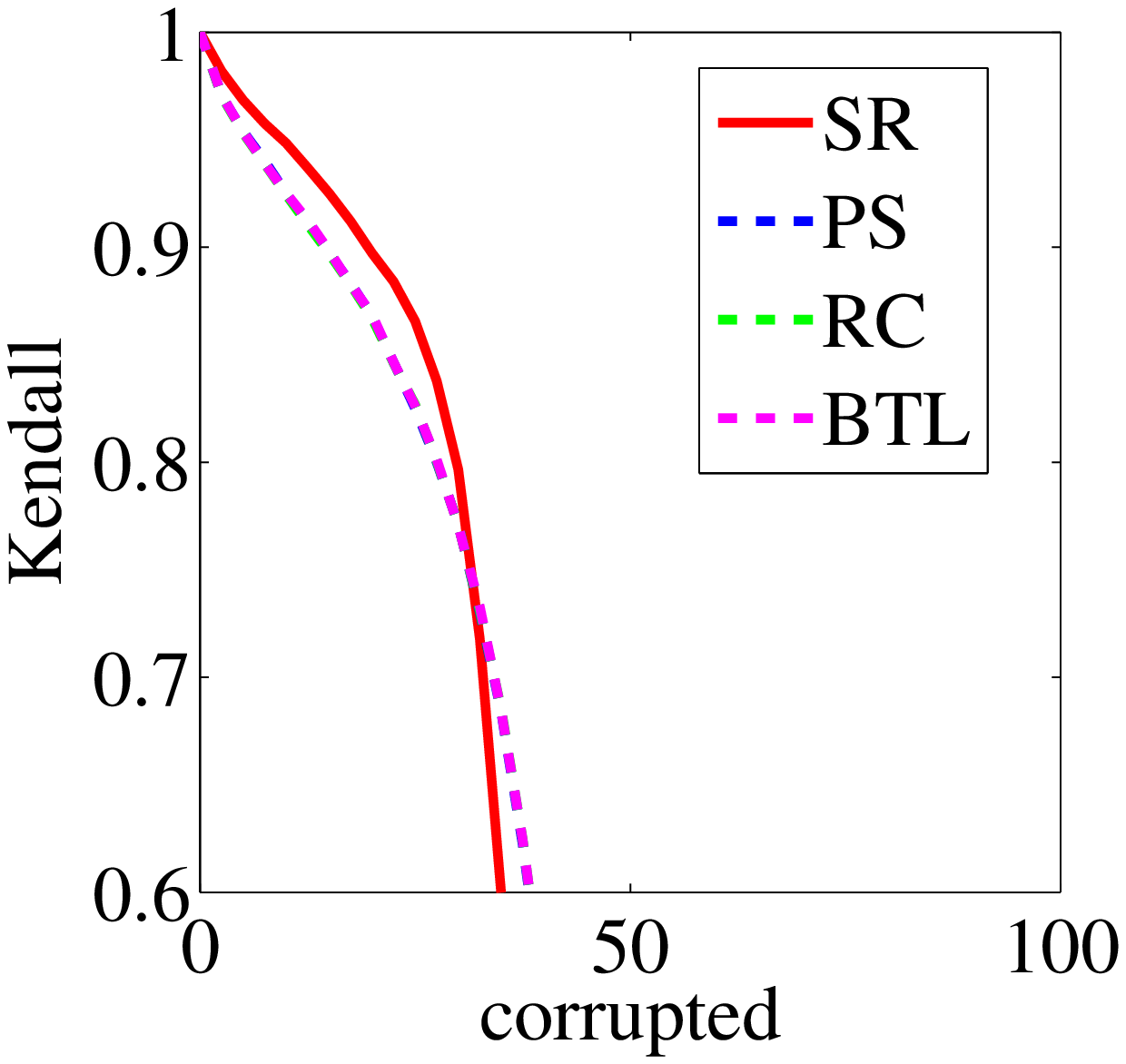} & &
\psfrag{Kendall}[b][t]{Kendall $\tau$}
\psfrag{missing}[t][b]{\% missing}
\includegraphics[scale=0.33]{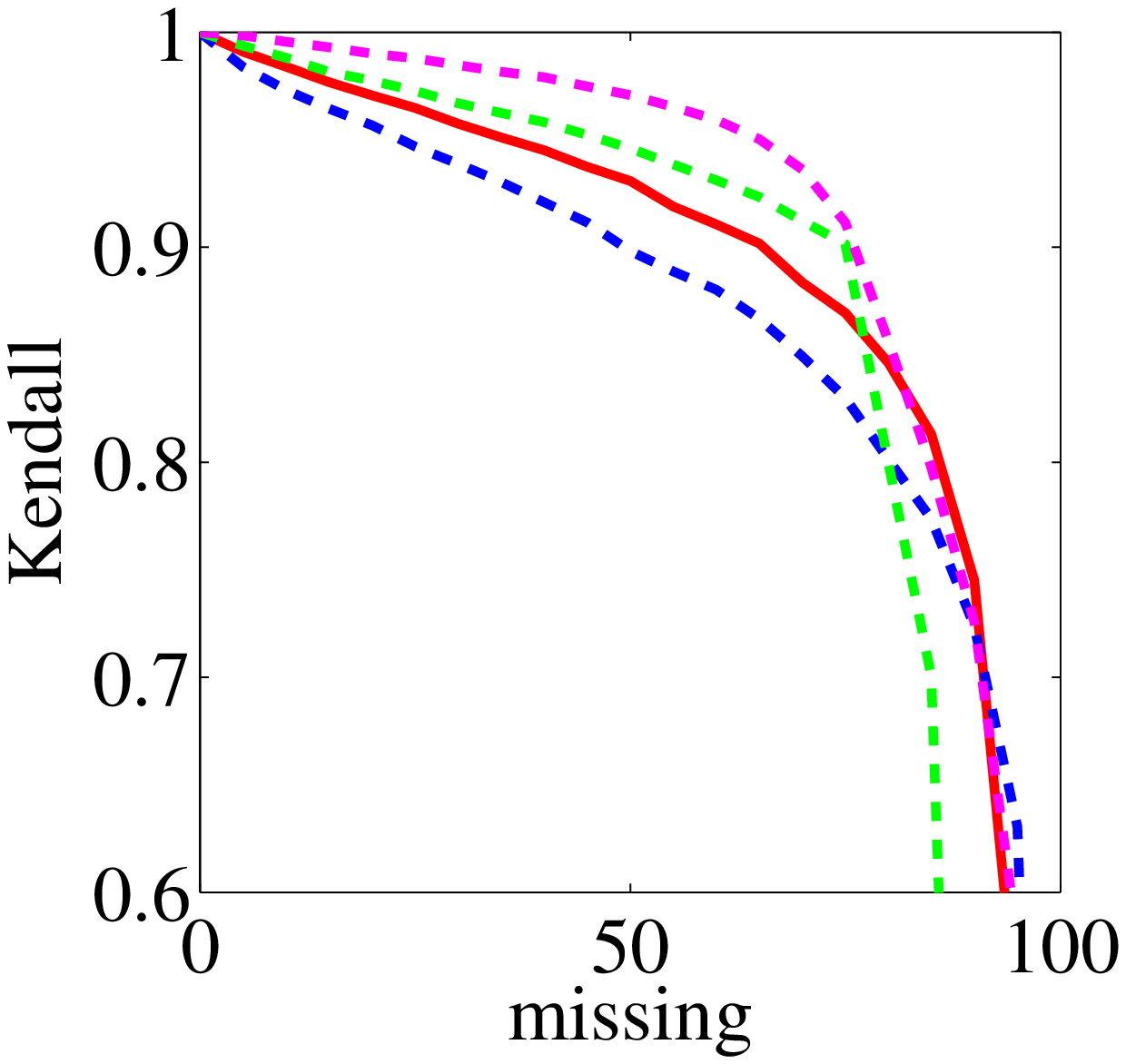} \\
\psfrag{Kendall}[b][t]{Kendall $\tau$}
\psfrag{missing}[t][b]{\% missing}
\includegraphics[scale=0.33]{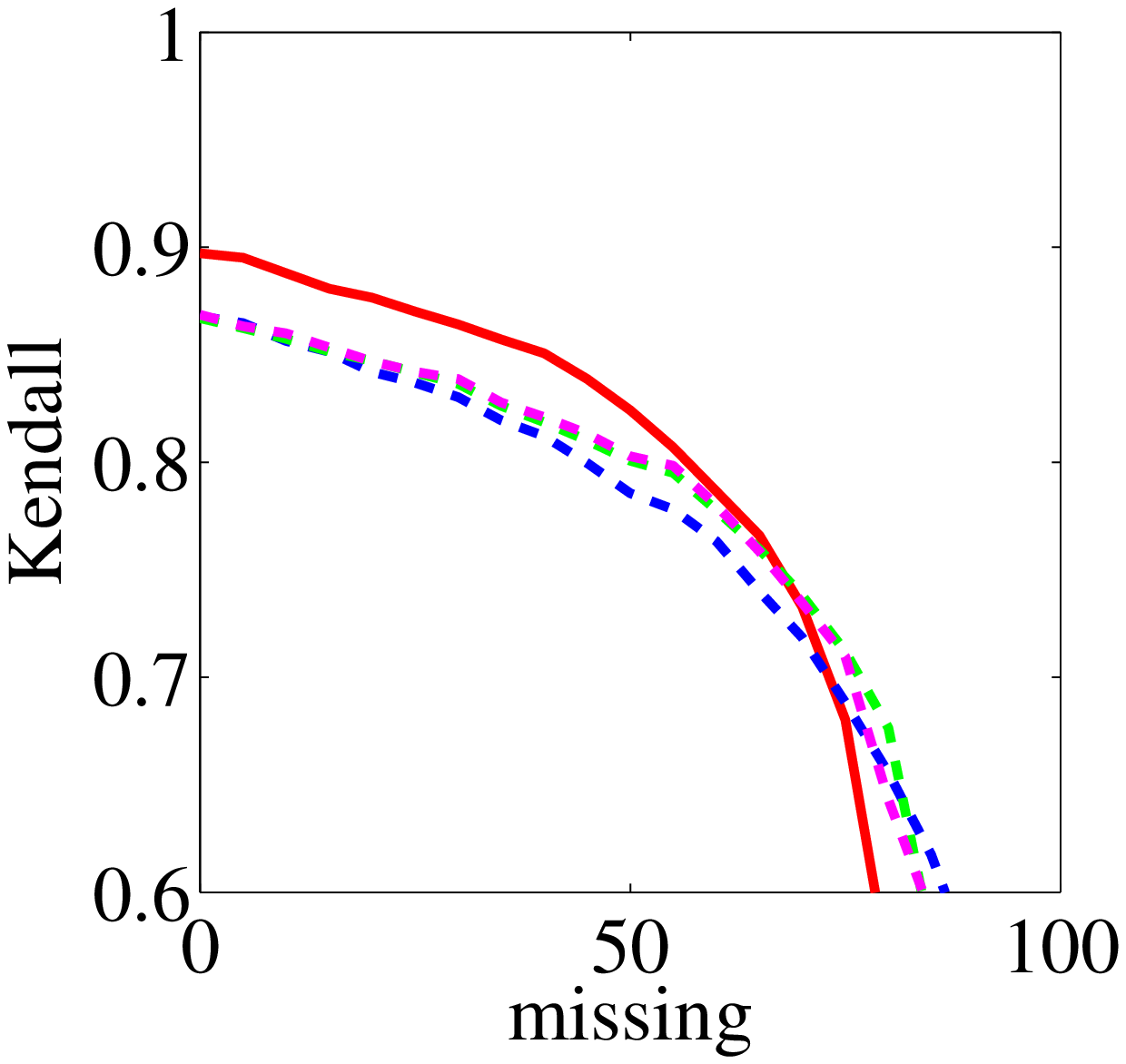} & &
\psfrag{Kendall}[b][t]{Kendall $\tau$}
\psfrag{m}[t][b]{Range $m$}
\includegraphics[scale=0.33]{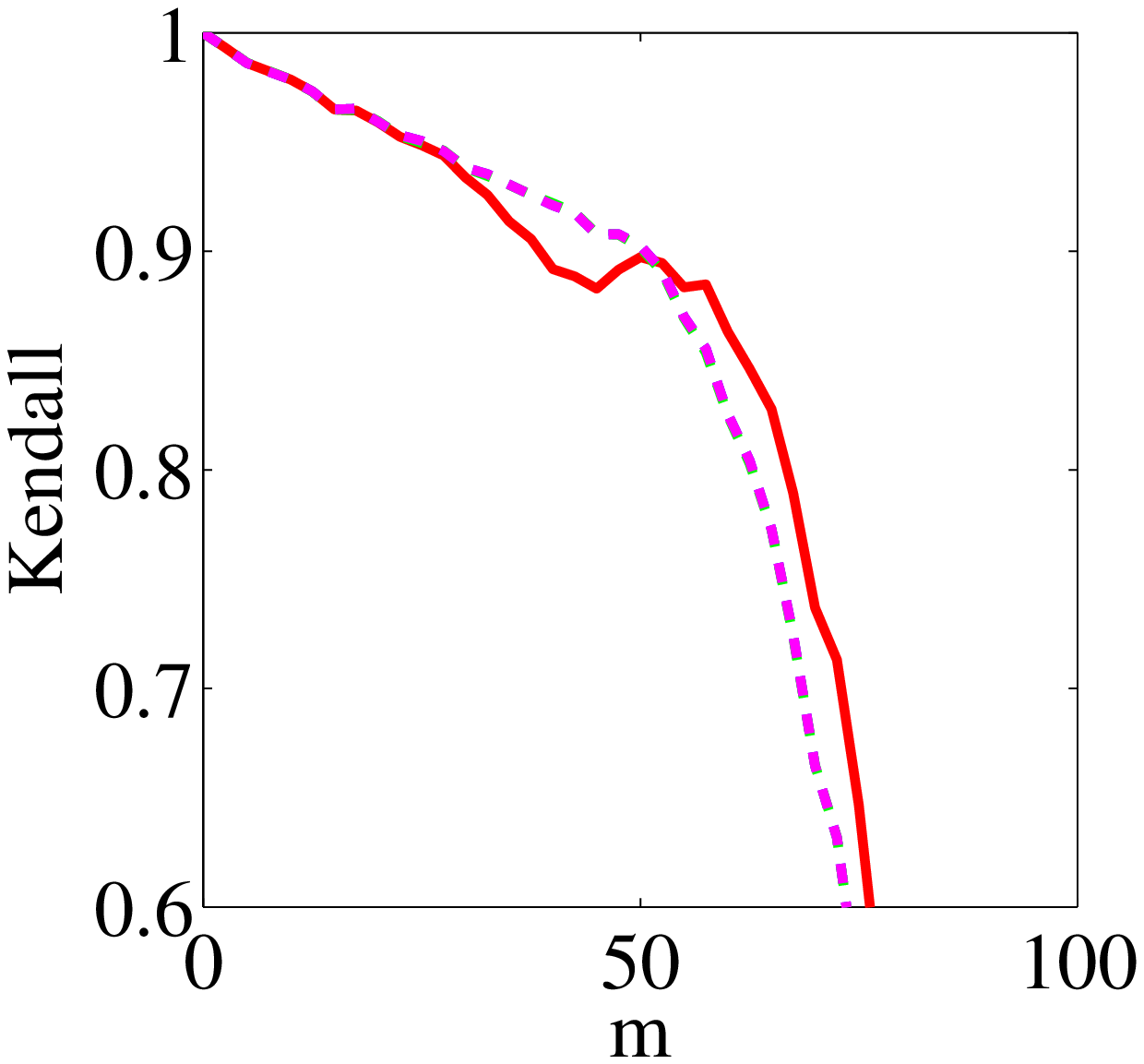}
\end{tabular}
\caption{Kendall $\tau$ (higher is better) for \ref{alg:SerialRank} (SR, full red line), point score (PS, \citep{Waut13} dashed blue line), Rank Centrality (RC \citep{Nega12} dashed green line), and maximum likelihood (BTL \citep{Brad52}, dashed magenta line). In the first synthetic dataset, we vary the proportion of corrupted comparisons {\em (top left)}, the proportion of observed comparisons {\em (top right)} and the proportion of observed comparisons, with 20\% of comparisons being corrupted {\em (bottom left)}. We also vary the parameter $m$ in the second synthetic dataset {\em (bottom right)}. \label{fig:impactNbObsNoise}}
\end{center}
 \end{figure}
 
\begin{figure}[t]
\begin{center}
\begin{tabular}{ccc} 
\psfrag{Dis}[b][t]{\% upsets in top $k$}
\psfrag{Top k}[t][b]{$k$}
\includegraphics[scale=0.45]{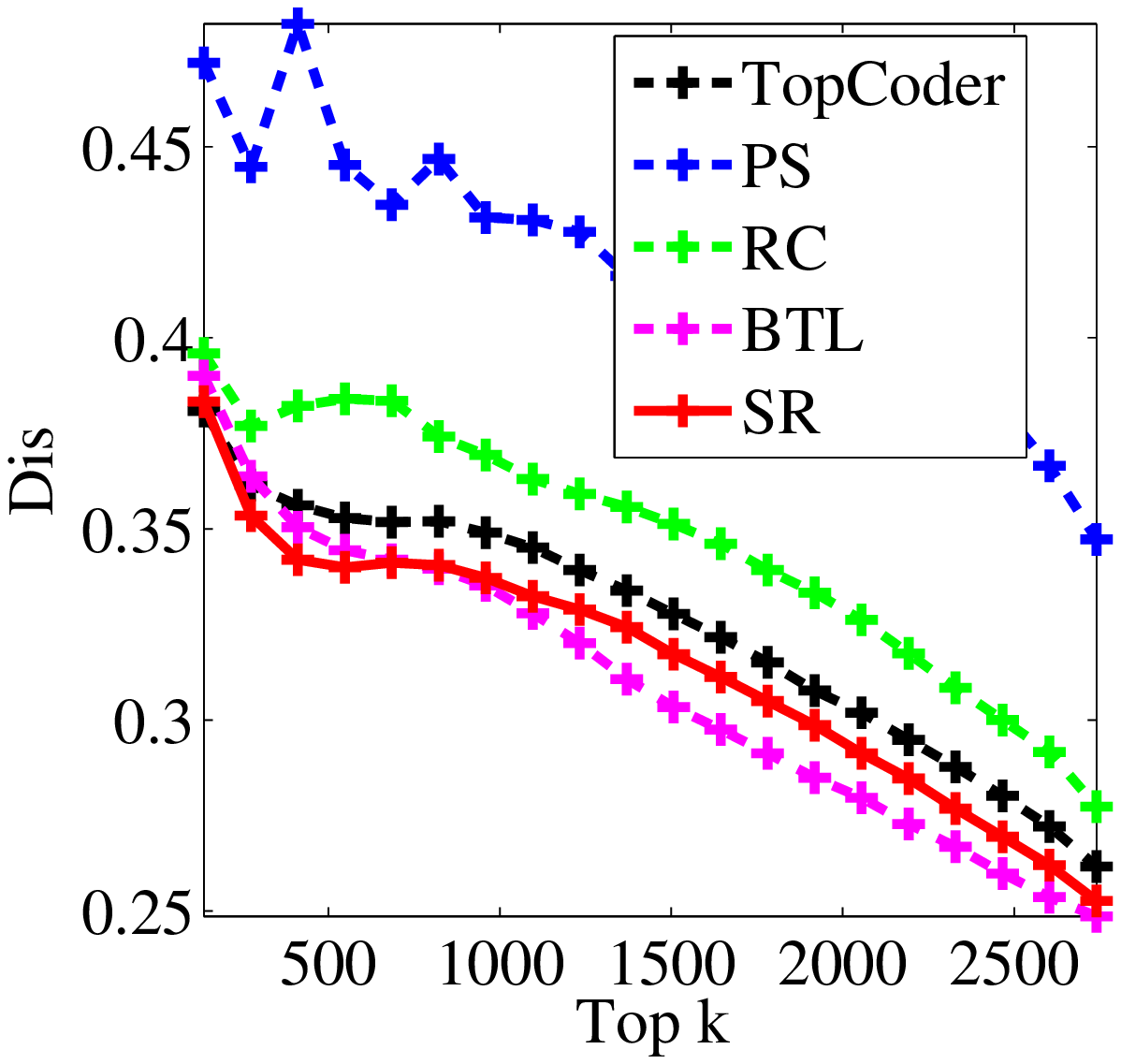} & &
\psfrag{Dis}[b][t]{\% upsets in top $k$}
\psfrag{Top k}[t][b]{$k$}
\includegraphics[scale=0.435]{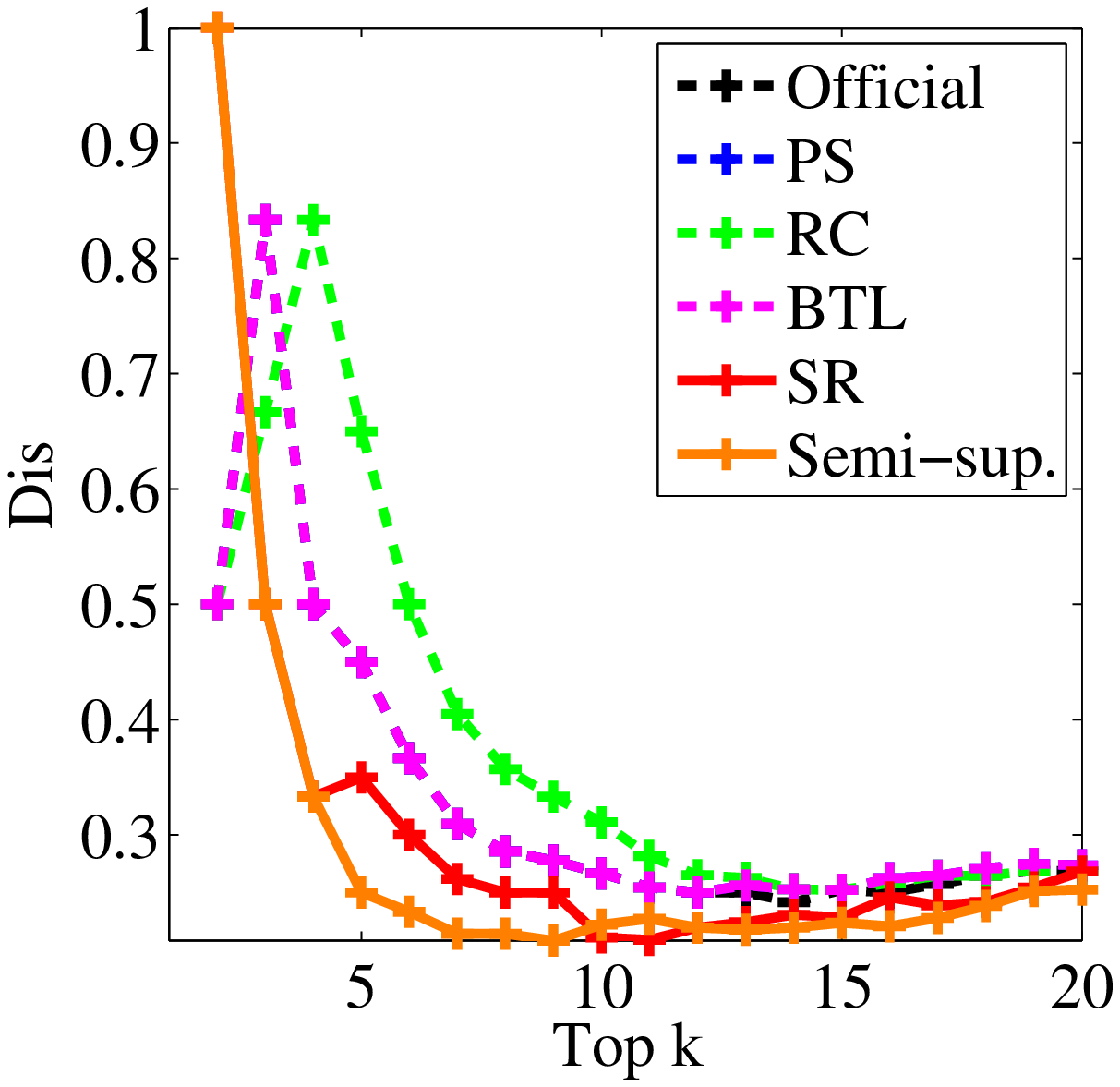}
\end{tabular}
\caption{Percentage of upsets (i.e. disagreeing comparisons, lower is better) defined in~\eqref{eq:loss}, for various values of $k$ and ranking methods, on TopCoder (\emph{left}) and football data (\emph{right}).
\label{fig:TopCoder}
}
\end{center}
\end{figure}

\begin{table}[h]
\begin{center}
\caption{Ranking of teams in the England premier league season 2013-2014.\label{tab:football}}
\vskip 0.5ex
{\scriptsize
\begin{tabular}{l|l|l|l|l|l}
\bf Official	&  \bf Row-sum 	&	\bf RC	&\bf 	BTL	&\bf 	SerialRank	&\bf 	Semi-Supervised\\\hline
Man City (86)	&	Man City	&	Liverpool	&	Man City	&	Man City	&	Man City		\\
Liverpool (84)	&	Liverpool	&	Arsenal	&	Liverpool	&	Chelsea	&	Chelsea		\\
Chelsea (82)	&	Chelsea	&	Man City	&	Chelsea	&	Liverpool	&	Liverpool		\\
Arsenal (79)	&	Arsenal	&	Chelsea	&	Arsenal	&	Arsenal	&	Everton		\\
Everton (72)	&	Everton	&	Everton	&	Everton	&	Everton	&	Arsenal		\\
Tottenham (69)	&	Tottenham	&	Tottenham	&	Tottenham	&	Tottenham	&	Tottenham	\\
Man United (64)	&	Man United	&	Man United	&	Man United	&	Southampton	&	Man United		\\
Southampton (56)	&	Southampton	&	Southampton	&	Southampton	&	Man United	&	Southampton		\\
Stoke (50)	&	Stoke	&	Stoke	&	Stoke	&	Stoke	&	Newcastle		\\
Newcastle (49)	&	Newcastle	&	Newcastle	&	Newcastle	&	Swansea	&	Stoke	\\
Crystal Palace (45)	&	Crystal Palace	&	Swansea	&	Crystal Palace	&	Newcastle	&	West Brom	\\
Swansea (42)	&	Swansea	&	Crystal Palace	&	Swansea	&	West Brom	&	Swansea		\\
West Ham (40)	&	West Brom	&	West Ham	&	West Brom	&	Hull	&	Crystal Palace	\\
Aston Villa (38)	&	West Ham	&	Hull	&	West Ham	&	West Ham	&	Hull	\\
Sunderland (38)	&	Aston Villa	&	Aston Villa	&	Aston Villa	&	Cardiff	&	West Ham		\\
Hull (37)	&	Sunderland	&	West Brom	&	Sunderland	&	Crystal Palace	&	Fulham	\\
West Brom (36)	&	Hull	&	Sunderland	&	Hull	&	Fulham	&	Norwich	\\
Norwich (33)	&	Norwich	&	Fulham	&	Norwich	&	Norwich	&	Sunderland		\\
Fulham (32)	&	Fulham	&	Norwich	&	Fulham	&	Sunderland	&	Aston Villa\\
Cardiff (30)	&	Cardiff	&	Cardiff	&	Cardiff	&	Aston Villa	&	Cardiff
\end{tabular}
}
\end{center}
\end{table}

\subsection{Semi-Supervised Ranking} We illustrate here how, in a semi-supervised setting, one can interactively enforce some constraints on the retrieved ranking, using e.g. the semi-supervised seriation algorithm in~\citep{Foge13}. We compute rankings of England Football Premier League teams for season 2013-2014 (\cf~figure~\ref{fig:FootballPastSeasons} for seasons 2011-2012 and 2012-2013). Comparisons are defined as the averaged outcome (win, loss, or tie) of home and away games for each pair of teams.
As shown in Table~\ref{tab:football}, the top half of \ref{alg:SerialRank} ranking is very close to the official ranking calculated by sorting the sum of points for each team (3 points for a win, 1 point for a tie). However, there are significant variations in the bottom half, though the number of upsets is roughly the same as for the official ranking. To test semi-supervised ranking, suppose for example that we are not satisfied with the ranking of Aston Villa (last team when ranked by the spectral algorithm), we can explicitly enforce that Aston Villa appears before Cardiff, as in the official ranking. In the ranking based on the corresponding semi-supervised seriation problem, Aston Villa is not last anymore, though the number of disagreeing comparisons remains just as low (\cf~Figure~\ref{fig:TopCoder}, \emph{right}).

\begin{figure}[t]
\begin{center}
\begin{tabular}{cc} 
\psfrag{Dis}[b][t]{\% upsets in top $k$}
\psfrag{Top k}[t][b]{$k$}
\includegraphics[scale=0.435]{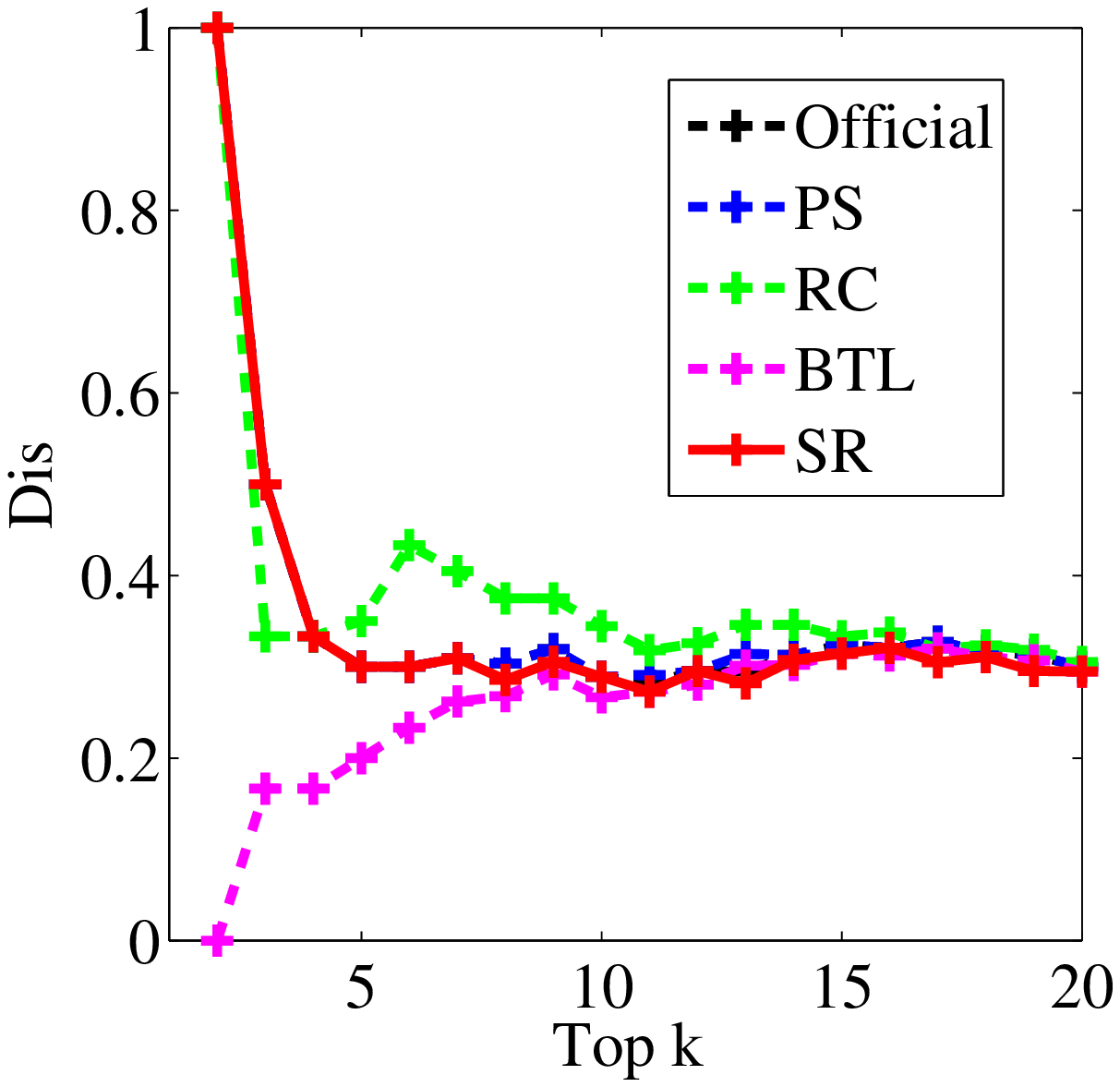} &
\psfrag{Dis}[b][t]{\% upsets in top $k$}
\psfrag{Top k}[t][b]{$k$}
\includegraphics[scale=0.435]{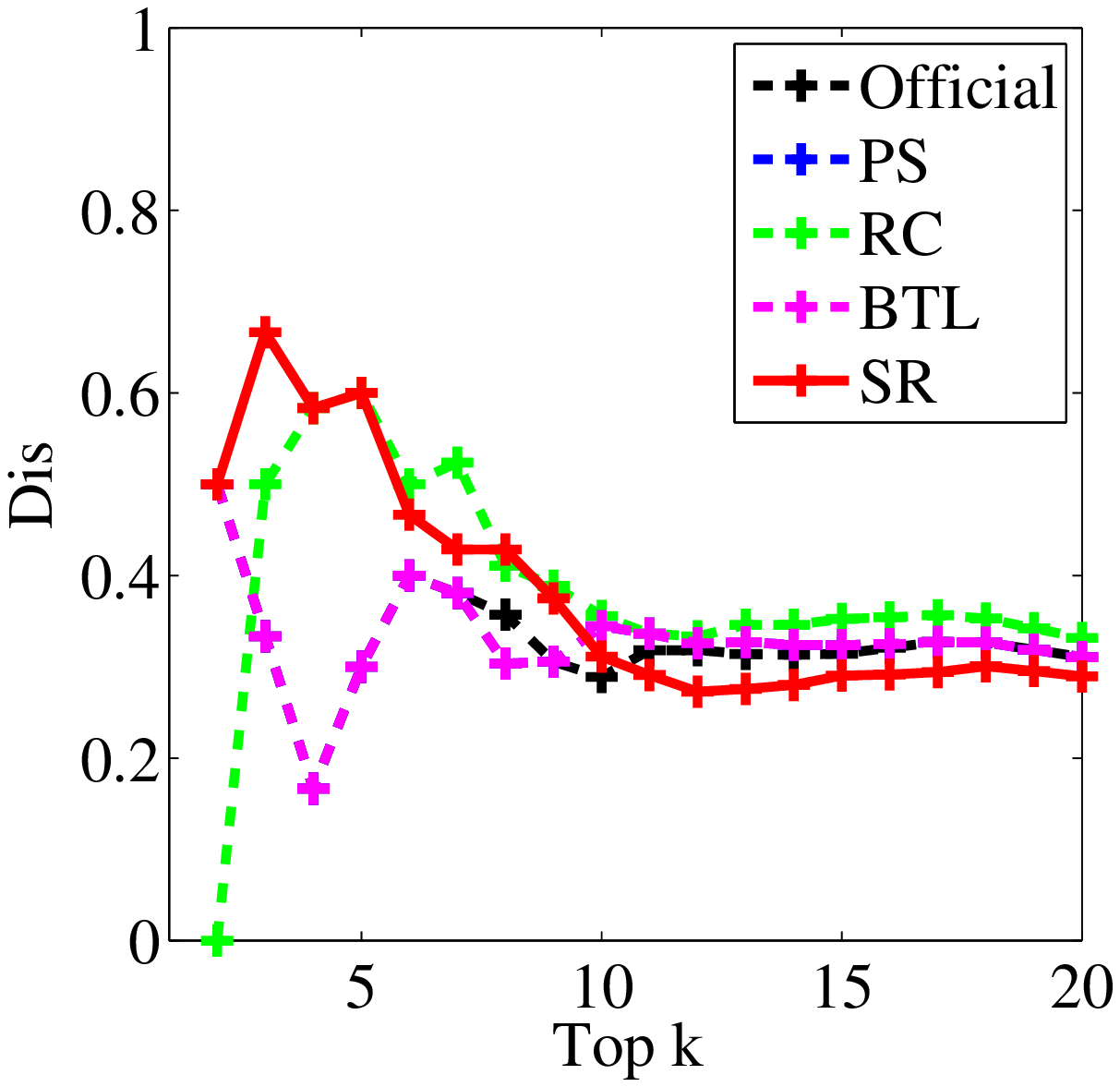} 
\end{tabular}
\caption{Percentage of upsets (i.e. disagreeing comparisons, lower is better) defined in~\eqref{eq:loss}, for various values of $k$ and ranking methods, on England Premier League 2011-2012 season (\emph{left}) and  2012-2013 season (\emph{right}).
\label{fig:FootballPastSeasons}
}
\end{center}
\end{figure}

\section{Conclusion}
We have formulated the problem of ranking from pairwise comparisons as a seriation problem, i.e. the problem of ordering from similarity information. By constructing an adequate similarity matrix, we applied a spectral relaxation for seriation to a variety of synthetic and real ranking datasets, showing competitive and in some cases superior performance compared to classical methods, especially in low noise environments. We derived performance bounds for this algorithm in the presence of corrupted and missing (ordinal) comparisons showing that SerialRank produces state-of-the art results for ranking based on ordinal comparisons, e.g. showing exact reconstruction w.h.p. when only $O(\sqrt{n})$ comparisons are missing. On the other hand, performance deteriorates when only a small fraction of comparisons are observed, or in the presence of very high noise. In this scenario, we showed that local ordering errors can be bounded if the number of samples is of order $O(n^{1.5}\mathrm{polylog}(n))$ which is significantly above the optimal bound of $O(n\log n)$.

A few questions thus remain open, which we pose as future research directions. First of all, from a theoretical perspective, is it possible to obtain an $\ell_\infty$ bound on local perturbations of the ranking using only $O(n\, \rm{polylog} (n))$ sampled pairs? Or, on the contrary, can we find a lower bound for spectral algorithms (i.e. perturbation arguments) imposing more than $\Omega(n\,\rm{polylog} (n))$ sampled pairs? Note that those questions hold for all current spectral ranking algorithms. 

Another line of research concerns the generalization of spectral ordering methods to more flexible settings, e.g., enforcing structural or a priori constraints on the ranking. Hierarchical ranking, i.e. running the spectral algorithm on increasingly refined subsets of the original data should be explored too. Early experiments suggests this works quite well, but no bounds are available at this point.

Finally, it would be interesting to investigate how similarity measures could be tuned for specific applications in order to improve SerialRank predictive power, for instance to take into account more information than win/loss in sports tournaments. Additional experiments in this vein can be found in \cite{cucuringu2015sync}. 

\newpage
{
\bibliographystyle{agsm}
\bibliography{/Users/Fajwel/Dropbox/PhD/SpectralRanking/Paper/Arxiv/MainPerso}
}

\section{Appendix}
\label{s:appendix}

We now detail several complementary technical results.

\subsection{Exact recovery results with missing entries}
Here, as in Section~\ref{s:robust}, we study the impact of one missing comparison on \ref{alg:SerialRank}, then extend the result to multiple missing comparisons. 

\begin{proposition}\label{prop:sup-miss}
Given pairwise comparisons $C_{s,t} \in \{-1,0,1\}$ between items ranked according to their indices, suppose only one comparison $C_{i,j}$ is missing, with $j-i>1$ (i.e., $C_{i,j}=0$), then $S^\mathrm{match}$ defined in~\eqref{eq:rnk-sim} remains \emph{strict-R} and the point score vector remains strictly monotonic.
\end{proposition}
\begin{proof}
We use the same proof technique as in Proposition~\ref{prop:sup-oneError}. We write the true score and comparison matrix $w$ and $C$, while the observations are written $\hat w$ and $\hat C$ respectively. This means in particular that $\hat C_{i,j}=0$. To simplify notations we denote by $S$ the similarity matrix $S^\mathrm{match}$ (respectively $\hat S$ when the similarity is computed from observations).
We first study the impact of the missing comparison $C_{i,j}$ for $i<j$ on the point score vector $\hat{w}$. We have
\[
\hat{w}_i = \sum_{k=1}^{n} \hat C_{k,i} = \sum_{k=1}^{n} C_{k,i} + \hat C_{j,i} - C_{j,i} = w_i + 1,
\]

similarly $\hat{w}_j=w_{j}-1$, whereas for $k\neq i,j$, $\hat{w_k}=w_k$.
Hence, $w$ is still strictly increasing if $j>i+1$. If $j=i+1$ there is a tie between $w_i$ and $w_{i+1}$. Now we show that the similarity matrix defined in \eqref{eq:rnk-sim} is an R-matrix. Writing $\hat S$ in terms of $S$, we get
\[
[\hat C\hat C^T]_{i,t}= \sum_{k\neq j} \left( \hat C_{i,k}\hat C_{t,k} \right) + \hat C_{i,j}\hat C_{t,j} 
= \sum_{k\neq j} \left( C_{i,k}C_{t,k} \right) = 
\left\{\BA{ll}
[CC^T]_{i,t} -1  & \mbox{if} \ t  <  j \\
\left[CC^T\right]_{i,t} +1 & \mbox{if} \ t > j.\\
\EA\right.
\]
We thus get
\[ \hat S_{i,t} =
 \left\{\BA{ll}
S_{i,t} - \frac{1}{2}  & \mbox{if} \ t  < j \\
S_{i,t} + \frac{1}{2} & \mbox{if} \ t > j,\\
\EA\right. \]
(remember there is a factor $1/2$ in the definition of $S$). Similarly we get for any $t \neq i$
\[ \hat S_{j,t} =
 \left\{\BA{ll}
S_{j,t} + \frac{1}{2}  & \mbox{if} \ t  <  i \\
S_{j,t} - \frac{1}{2} & \mbox{if} \ t > i.\\
\EA\right. \]
Finally, for the single corrupted index pair $(i,j)$, we get
\[
\hat S_{i,j}= \frac{1}{2} \left( n+ \sum_{k\neq i,j} \left( \hat C_{i,k}\hat C_{j,k} \right) + \hat C_{i,i}\hat C_{j,i} + \hat C_{i,j}\hat C_{j,j} \right)
= S_{i,j}  - 0 +0 = S_{i,j}.
\]
For all other coefficients $(s,t)$ such that $s,t \neq i,j$, we have $\hat S_{s,t}=S_{s,t}$. Meaning all rows or columns outside of $i,j$ are left unchanged. We first observe that these last equations, together with our assumption that $j-i>2$, mean that
\[
\hat S_{s,t} \geq \hat S_{s+1,t} \quad \mbox{and} \quad \hat S_{s,t+1} \geq \hat S_{s,t},\quad \mbox{for any } s < t
\]
so $\hat S$ remains an R-matrix.
To show uniqueness of the retrieved order, we need $j-i>1$. Indeed, when $j-i>1$ all these $R$ constraints are strict, which means that $\hat S$ is still a strict R-matrix, hence the desired result.
\end{proof}

We can extend this result to the case where multiple comparisons are missing.

\begin{proposition}\label{prop:sup-kmiss}
Given pairwise comparisons $C_{s,t} \in \{-1,0,1\}$ between items ranked according to their indices, suppose $m$ comparisons indexed $(i_1,j_1),\ldots,(i_m,j_m)$ are missing, i.e., $C_{i_l,j_j}=0$ for $i=l,\ldots,m$. If the following condition~\eqref{cond:distrkmissing} holds true,
\BEQ \label{cond:distrkmissing} 
| s - t | > 1 \ \mathrm{for \ all} \ s\neq t \in \{i_1,\dots,i_{m},j_1,\dots,j_{m}\}
\EEQ  
 then $S^\mathrm{match}$ defined in~\eqref{eq:rnk-sim} remains \emph{strict-R} and the point score vector remains strictly monotonic.
\end{proposition}
\begin{proof}
Proceed similarly as in the proof of Proposition~\ref{prop:sup-kerrors}, except that shifts are divided by two.
\end{proof}

We also get the following corollary.

\begin{corollary}
\label{cor:impactCorMiss}
Given pairwise comparisons $C_{s,t} \in \{-1,0,1\}$ between items ranked according to their indices, suppose $m$ comparisons indexed $(i_1,j_1),\ldots,(i_m,j_m)$ are either corrupted or missing. If condition~\eqref{cond:app-distrkerrors} holds true then $S^\mathrm{match}$ defined in~\eqref{eq:rnk-sim} remains \emph{strict-R}.
\end{corollary}
\begin{proof}
Proceed similarly as the proof of Proposition~\ref{prop:sup-kerrors}, except that shifts are divided by two for missing comparisons.
\end{proof}

\subsection{Standard theorems and technical lemmas used in spectral perturbation analysis (section~\ref{s:spectralPertAnalysis})}
\label{s:appendixPertAnalysis}

We first recall Weyl's inequality and a simplified version of Davis-Kahan theorem which can be found in \citep{Stew90,Stew01,yu2015}.

\begin{theorem}{\bf (Weyl's inequality)}
\label{lemma:Weyl}
Consider a symmetric matrix $A$ with eigenvalues $\lambda_1,\ldots,\lambda_n$ and $\tilde A$ a symmetric perturbation of $A$ with eigenvalues $\tilde \lambda_1,\ldots,\tilde \lambda_n$,
\[
\max_i |\tilde \lambda_i - \lambda_i| \leq \|\tilde A -A \|_2.
\]

\end{theorem}
\begin{theorem} {\bf (Variant of Davis-Kahan theorem \citep[Corollary 3][]{yu2015})} 
\label{lemma:davis-kahan}
Let $A, \tilde A \in \reals^n$ be symmetric, with eigenvalues $\lambda_1 \leq \ldots \leq \lambda_n$ and $\tilde \lambda_1 \leq \ldots \leq \tilde \lambda_n$ respectively. Fix $j \in \{1,\ldots, n\}$, and assume that $\min(\lambda_j - \lambda_{j-1}, \lambda_{j+1} - \lambda_j) > 0$, where
$\lambda_{n+1} := \infty$ and $\lambda_{0} := -\infty$. 
If $v, \tilde v \in \reals^n$ satisfy $Av = \lambda_j v$ and $\tilde A \tilde v = \tilde \lambda_j \tilde v$, then
\[
\sin \Theta(\tilde v,v)\leq \frac{2\|\tilde A-A\|_2}{\min(\lambda_{j}-\lambda_{j-1}, \lambda_{j+1}-\lambda_{j})}.
\]
Moreover, if $\tilde v^T v \geq 0$, then
\[
\| \tilde v-v \|_2\leq \frac{2\sqrt{2}\|\tilde A-A\|_2}{\min(\lambda_{j}-\lambda_{j-1}, \lambda_{j+1}-\lambda_{j})}.
\]
\end{theorem}
When analyzing the perturbation of the Fiedler vector $f$, we may always reverse the sign of $\tilde f$ such that $\tilde f^T f \geq 0$ and obtain
\[
\| \tilde f-f \|_2 \leq \frac{2\sqrt{2}\|\tilde L-L\|_2}{\min(\lambda_{2}-\lambda_{1}, \lambda_{3}-\lambda_{2})}.
\]

\lemmaboundRf*

\begin{proof}
The proof is very much similar to the proof of Lemma~\ref{lemma:boundDegree}.
Let $R=\tilde S- S$.
We have
\[
R_{ij}=\sum_{k=1}^n C_{ik}C_{jk} \left(\frac{B_{ik}B_{jk}}{q^2(2p-1)^2}-1\right).
\]
Therefore, let $\delta=Rf$
\[
\delta_i=\sum_{j=1}^n R_{ij} f_j= \sum_{j=1}^n \sum_{k=1}^n C_{ik}C_{jk} \left(\frac{B_{ik}B_{jk}}{q^2(2p-1)^2}-1 \right)f_j.
\]
Notice that we can arbitrarily fix the diagonal values of $R$ to zeros. Indeed, the similarity between an element and itself should be a constant by convention, which leads to $R_{ii}=\tilde S_{ii}- S_{ii}=0$ for all items $i$. Hence we could take $j\neq i $ in the definition of $d_i$, and we can consider $B_{ik}$ independent of $B_{jk}$ in the associated summation.

We first obtain a concentration inequality for each $\delta_i$. We will then use a union bound to bound $\|\delta \|_{\infty}=\max |\delta_i|$. Notice that
\begin{eqnarray*}
\delta_i &=& \sum_{j=1}^n \sum_{k=1}^n C_{ik}C_{jk} \left(\frac{B_{ik}B_{jk}}{q^2(2p-1)^2}-1\right)f_j \\
&=& \sum_{k=1}^n \left( \frac{C_{ik}B_{ik}}{q(2p-1)} \sum_{j=1}^n C_{jk} \left(\frac{B_{jk}}{q(2p-1)}-1\right)f_j \right)
+  \sum_{k=1}^n \sum_{j=1}^n C_{ik}C_{jk} \left(\frac{B_{ik}}{q(2p-1)}-1 \right)f_j. 
\end{eqnarray*}
The first term is quadratic while the second is linear, both terms have mean zero since the $B_{ik}$ are independent of  the $B_{jk}$.
We begin by bounding the quadratic term. Let $X_{jk}= C_{jk} (\frac{1}{q(2p-1)}B_{jk}-1)f_j$.
We have 
$$\Expect(X_{jk})=f_jC_{jk} (\frac{qp-q(1-p)}{q(2p-1)}-1)=0,$$

$$\var(X_{jk})=\frac{f_j^2\var(B_{jk})}{q^2(2p-1)^2}=\frac{f_j^2}{q^2(2p-1)^2}(q-q^2(2p-1)^2)\leq  \frac{f_j^2}{q(2p-1)^2},$$

$$|X_{jk}|=|f_j| |\frac{B_{jk}}{q(2p-1)}-1|  \leq  \frac{2|f_j|}{q(2p-1)}\leq  \frac{2\|f\|_{\infty}}{q(2p-1)^2}.$$
From corollary~\ref{cor:boundMaxFiedler} $\|f\|_{\infty}\leq 2/\sqrt{n}$. Moreover $\sum_{j=0}^n f_j^2=1$ since $f$ is an eigenvector.
Hence, by applying Bernstein inequality we get for any $t>0$
\BEQ
\label{eq:Bernstein5}
\Prob \left(|\sum_{j=1}^n X_{jk}|>t \right)\leq 2 \exp \left( \frac{- q(2p-1)^2t^2}{2(1+2 t/(3\sqrt{n}))} \right)\leq 2 \exp \left( \frac{- q(2p-1)^2t^2n}{2(n+\sqrt{n} t)} \right).
\EEQ
The rest of the proof is identical to the proof of Lemma~\ref{lemma:boundDegree}, replacing $t$ by $\sqrt{n}t$.
\end{proof}

\lemmaboundLinfFiedler*
\begin{proof}
Notice that by definition
$\tilde L \tilde f =\tilde \lambda_2 \tilde f$ and $Lf=\lambda_2 f$. Hence for $\tilde \lambda_2 >0$
\begin{eqnarray*}
\tilde f -f &=& \frac{\tilde L \tilde f}{\tilde \lambda_2} -f \\
&=& \frac{\tilde L \tilde f - Lf}{\tilde \lambda_2} + \frac{(\lambda_2 - \tilde \lambda_2)f}{\tilde \lambda_2}.
\end{eqnarray*}
Moreover
\begin{eqnarray*}
\tilde L \tilde f - Lf & = & (\idm - \tilde D^{-1}\tilde S)\tilde f - (\idm -D^{-1}S)f \\
& = & (\tilde f  -  f) + D^{-1} S f - \tilde D^{-1}\tilde S \tilde f \\
& = & (\tilde f  -  f) + D^{-1} S f - \tilde D^{-1}\tilde S  f + \tilde D^{-1}\tilde S  f - \tilde D^{-1}\tilde S  \tilde f \\
& = & (\tilde f  -  f) + (D^{-1} S - \tilde D^{-1}\tilde S)f + \tilde D^{-1}\tilde S (f - \tilde f) 
\end{eqnarray*}
Hence
\BEQ
(\idm(\tilde \lambda_2-1)+\tilde D^{-1}\tilde S)(\tilde f-f)  =  (D^{-1} S- \tilde D^{-1}\tilde S+ (\lambda_2 - \tilde \lambda_2)\idm)f .
\EEQ
Writing $S_i$ the $i^{th}$ row of $S$ and $d_i$ the degree of row $i$, using the triangle inequality, we deduce that
\BEQ
| \tilde f_i-f_i | \leq  \frac{1}{|\tilde \lambda_2 -1|}\left(|(d_i^{-1} S_i- \tilde d_i^{-1}\tilde S_i) f| + |\lambda_2 - \tilde \lambda_2| |f_i| + |\tilde d_i^{-1}\tilde S_i (\tilde f-f)| \right).
\EEQ

We will now bound each term separately. Define
\[\BA{l}
\mbox{Denom} =|\tilde \lambda_2 -1|,\\
\mbox{Num1} =|(d_i^{-1} S_i- \tilde d_i^{-1}\tilde S_i) f|,\\
\mbox{Num2} = |\lambda_2 - \tilde \lambda_2| |f_i|,\\
\mbox{Num3} =|\tilde d_i^{-1}\tilde S_i (\tilde f-f)|. 
\EA\]

\paragraph{\textbf{Bounding Denom}}
First notice that using Weyl's inequality and equation \eqref{eq:boundLaplacianPert}  (\cf~proof of Theorem~\ref{thm:boundFiedlerPerturbation}), we have with probability at least $1-2/n$
$|\tilde \lambda_2 -\lambda_2| \leq \| L_R\|_2 \leq \frac{c \mu}{\sqrt{\log n}}$. Therefore there exists an absolute constant $c$ such that with probability at least $1-2/n$
\[
|\tilde \lambda_2 -1| > c.
\]
We now proceed with the numerator terms.
\paragraph{\textbf{Bounding Num2}}
Using Weyl's inequality, corollary~\ref{cor:boundMaxFiedler} and equation \eqref{eq:boundLaplacianPert}  (\cf~proof of Theorem~\ref{thm:boundFiedlerPerturbation}), we deduce that w.h.p.
\[
|\lambda_2 - \tilde \lambda_2| |f_i| | \leq \frac {c \mu }{\sqrt{n \log n}},
\]
where $c$ is an absolute constant.
\paragraph{\textbf{Bounding Num1}}
We now bound $|d_i^{-1} S_i- \tilde d_i^{-1}\tilde S_i|$.
We have
\begin{eqnarray*}
|(\tilde d_i^{-1}\tilde S_i - d_i^{-1} S_i )f|
& = & |(\tilde d_i^{-1}\tilde S_i - \tilde d_i^{-1}S_i  +  \tilde d_i^{-1}S_i  -  d_i^{-1}S_i)f | \\
& \leq & |\tilde d_i^{-1}| |(\tilde S_i -S_i)f| + |(\tilde d_i^{-1}-d_i^{-1})S_i f|.
\end{eqnarray*}
Using equation \eqref{eq:boundPerturbDegree} from the proof of Theorem~\ref{thm:boundFiedlerPerturbation}, we have w.h.p.
$|\tilde d_i^{-1}-d_i^{-1}|\leq \frac{c\mu}{n^2\sqrt{\log n}}$. Moreover
\[
|\tilde d_i^{-1}|\leq  |\tilde d_i^{-1} - d_i^{-1}| + |d_i^{-1}| \leq \frac{c_1 \mu }{n^2\sqrt{\log n}} + \frac{c_2}{n^2} \leq \frac{c}{n^2}
\] w.h.p., where $c$ is an absolute constant.
Therefore
\begin{eqnarray}
\label{eq:boundNum1.1}
|(\tilde d_i^{-1}\tilde S_i - d_i^{-1} S_i )f|
& \leq &  \frac{c \mu}{n^2\sqrt{\log n}} |S_i f| + \frac{c}{n^2}  |(\tilde S_i -S_i)f|  \mbox{ w.h.p.}
\end{eqnarray}
Using the definition of $S$ and corollary~\ref{cor:boundMaxFiedler}, we get
\BEQ
\label{eq:boundNum1.2}
 |S_i f| \leq \sum_{j=1}^{n} S_{ij} \max_i |f_i| \leq c \frac{n^2}{\sqrt{n}} \leq  c n^{3/2},
\EEQ
where $c$ is an absolute constant. Using Lemma~\ref{lemma:boundRf}, we get 
\BEQ
\label{eq:boundNum1.3}
|(\tilde S_i -S_i)f| \leq \frac {3 \mu n^{3/2} }{\sqrt{\log n}} \mbox{ w.h.p.}
\EEQ 
Combining \eqref{eq:boundNum1.1}, \eqref{eq:boundNum1.2} and \eqref{eq:boundNum1.3} we deduce that there exists a constant $c$ such that 
\[
|(\tilde d_i^{-1}\tilde S_i - d_i^{-1} S_i )f| \leq \frac {c \mu }{\sqrt{n \log n}} \mbox{ w.h.p.}
\]

\paragraph{\textbf{Bounding Num3}}
Finally we bound the remaining term $|\tilde d_i^{-1}\tilde S_i (\tilde f-f)|$.
By Cauchy-Schwartz inequality we have,
\[
|\tilde d_i^{-1}\tilde S_i (\tilde f-f)| \leq |\tilde d_i^{-1}| \|\tilde S_i\|_2 \|\tilde f -f\|_2.
\]
Notice that
\[
\|\tilde S_i\|_2 \leq \|S_i\|_2 + \|\tilde S_i - S_i\|_2 \leq  \|S_i\|_2 + \|\tilde S - S\|_2.
\]
Since $\|S_i\|_2^2 \leq \|S_1\|_2^2 \leq \frac{n(n+1)(2n+1)}{6}$ and $q>\frac{\log^4 n}{\mu^2(2p-1)^2 \sqrt{n}}$ we deduce from Lemma~\ref{lemma:boundSimilarityPert} that w.h.p. $\|\tilde S_i\|_2 \leq \frac {c \mu n^{7/4} }{\sqrt{\log n}}$, where $c$ is an absolute constant, for $n$ large enough.
Moreover, as shown above, $|\tilde d_i^{-1}|\leq \frac{c}{n^2}$ and we also get from Theorem~\ref{thm:boundFiedlerPerturbation} that $\|\tilde f -f\|_2 \leq   \frac{c\mu}{n^{1/4}\sqrt{\log n}} $ w.h.p.
Hence we have
\[
|\tilde d_i^{-1}\tilde S_i (\tilde f-f)| \leq \frac {c \mu^2 n^{7/4} }{n^2 n^{1/4}(\log n)}
\leq \frac {c \mu}{\sqrt{n \log n}} \mbox{ w.h.p.},
\]
where $c$ is an absolute constant. Combining bounds on the denominator and numerator terms yields the desired result.
\end{proof}

\subsection{Numerical experiments with normalized Laplacian}

As shown in figure~\ref{fig:impactNbObsNoiseNormalized}, results are very similar to those of \ref{alg:SerialRank} with unnormalized Laplacian. We lose a bit of performance in terms of robustness to corrupted comparisons.

\begin{figure}[h]
\begin{center}
\begin{tabular}{ccc}
\psfrag{Kendall}[b][t]{Kendall $\tau$}
\psfrag{corrupted}[t][b]{\% corrupted}
\includegraphics[scale=0.33]{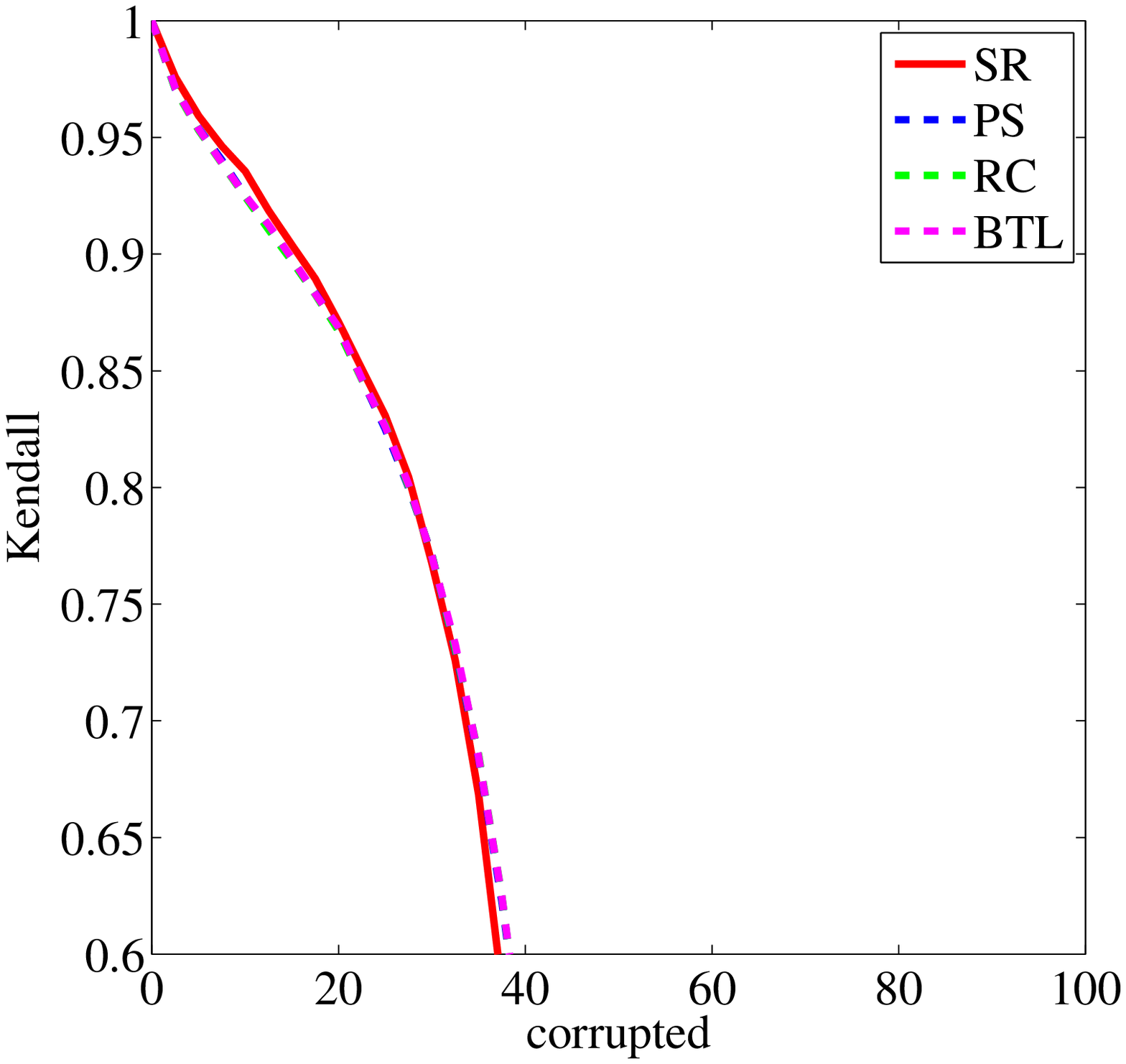} & &
\psfrag{Kendall}[b][t]{Kendall $\tau$}
\psfrag{missing}[t][b]{\% missing}
\includegraphics[scale=0.33]{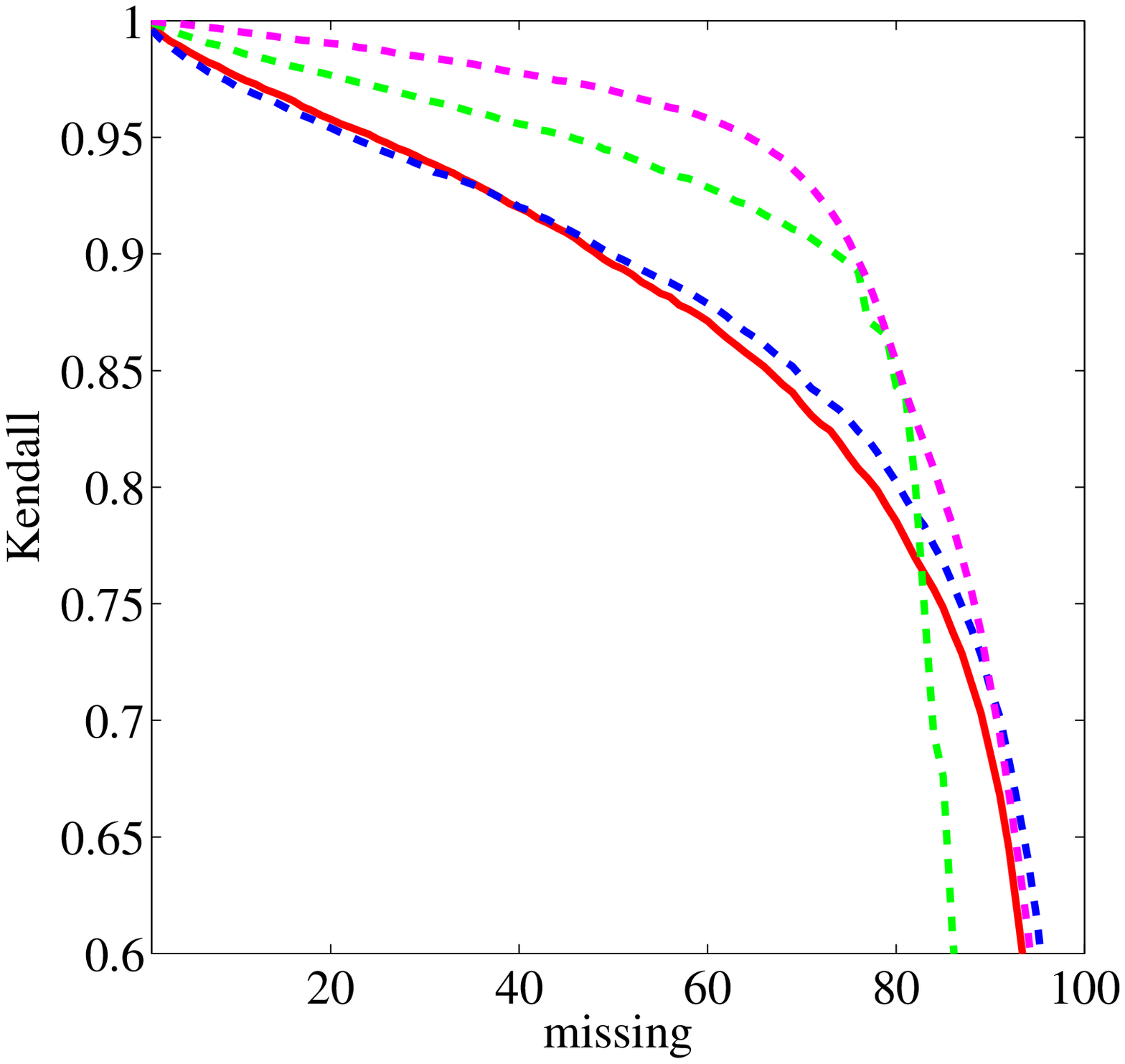} \\
\psfrag{Kendall}[b][t]{Kendall $\tau$}
\psfrag{missing}[t][b]{\% missing}
\includegraphics[scale=0.33]{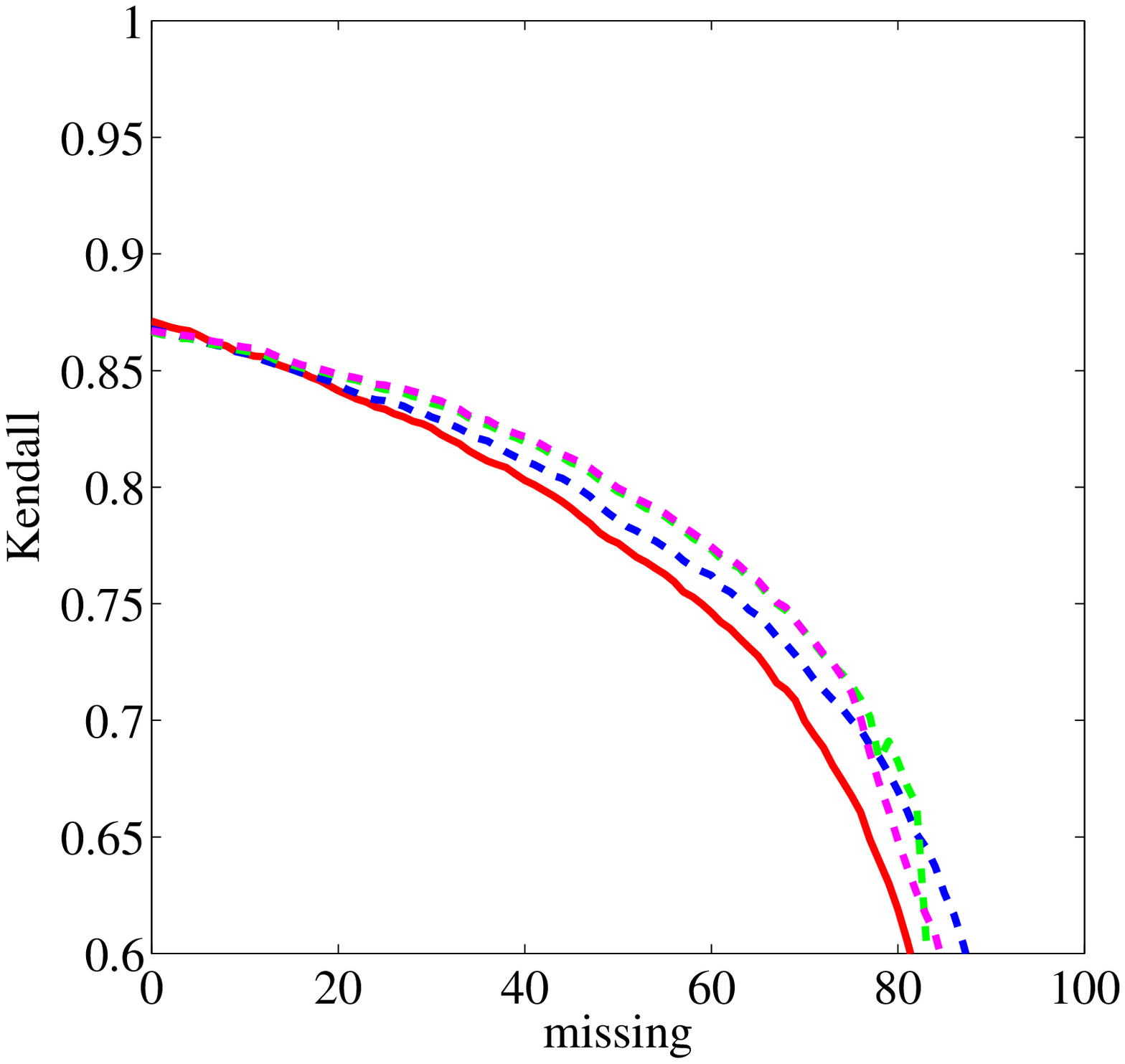} & &
\psfrag{Kendall}[b][t]{Kendall $\tau$}
\psfrag{m}[t][b]{Range $m$}
\includegraphics[scale=0.33]{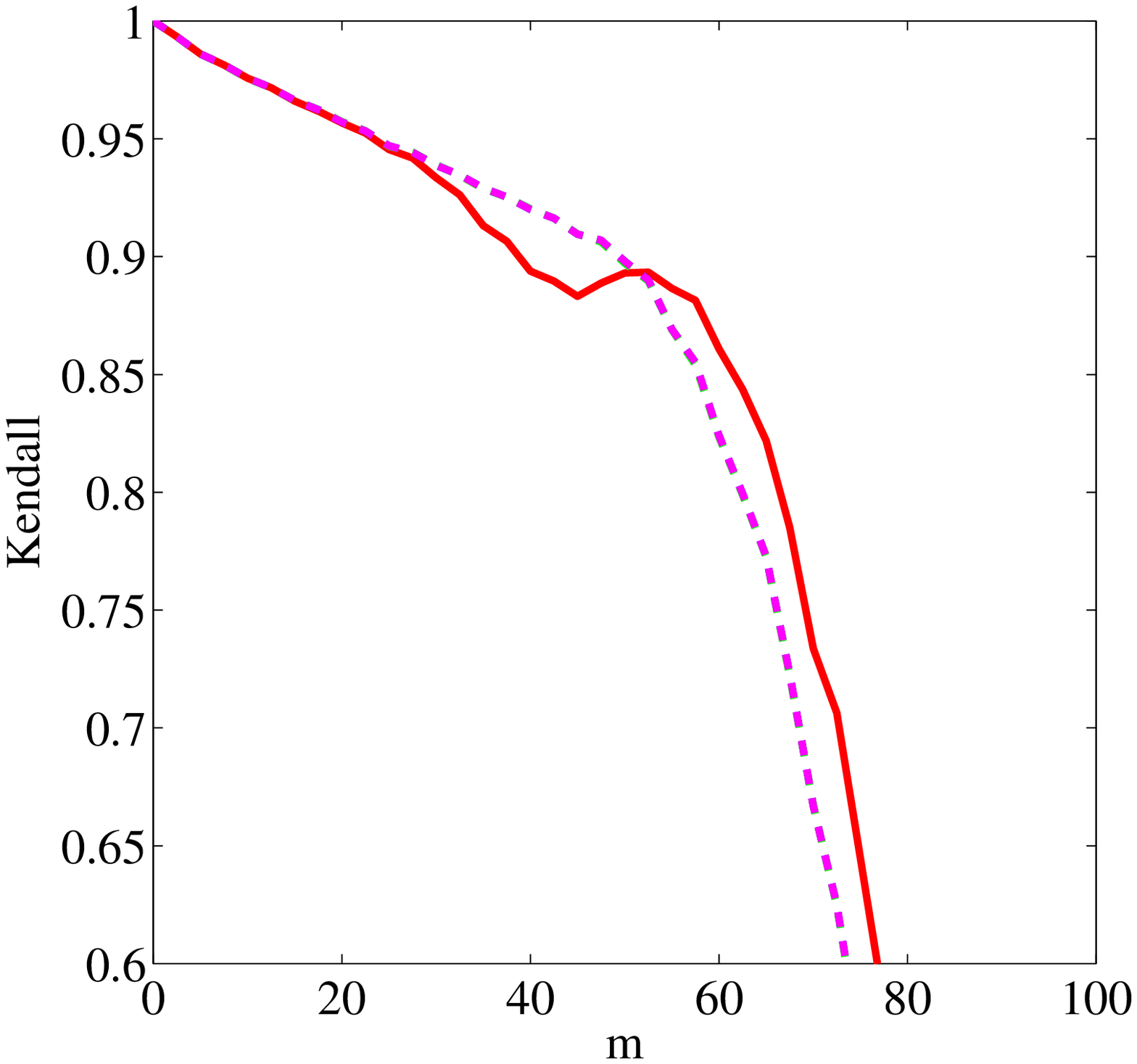}
\end{tabular}
\caption{Kendall $\tau$ (higher is better) for \ref{alg:SerialRank} with normalized Laplacian (SR, full red line), row-sum (PS, \citep{Waut13} dashed blue line), rank centrality (RC \citep{Nega12} dashed green line), and maximum likelihood (BTL \citep{Brad52}, dashed magenta line). In the first synthetic dataset, we vary the proportion of corrupted comparisons {\em (top left)}, the proportion of observed comparisons {\em (top right)} and the proportion of observed comparisons, with 20\% of comparisons being corrupted {\em (bottom left)}. We also vary the parameter $m$ in the second synthetic dataset {\em (bottom right)}. \label{fig:impactNbObsNoiseNormalized}}
\end{center}
 \end{figure}
 
\subsection{Spectrum of the unnormalized Laplacian matrix}

\subsubsection{Asymptotic Fiedler value and Fiedler vector}

We use results on the convergence of Laplacian operators to provide a description of the spectrum of the unnormalized Laplacian in \ref{alg:SerialRank}. Following the same analysis as in \citep{Von-08} we can prove that asymptotically, once normalized by $n^2$, apart from the first and second eigenvalue, the spectrum of the Laplacian matrix is contained in the interval $[0.5, 0.75]$. Moreover, we can characterize the eigenfunctions of the limit Laplacian operator by a differential equation, enabling to have an asymptotic approximation for the Fiedler vector. 

Taking the same notations as in \citep{Von-08}
we have here $k(x,y)=1-|x-y|$.
The degree function is
\[
d(x)=\int_0^1 k(x,y) d\Prob(y)=\int_0^1 k(x,y) d(y)
\]
(samples are uniformly ranked).
Simple calculations give
\[
d(x)=-x^2+x+1/2.
\]
We deduce that the range of $d$ is $[0.5,0.75]$. Interesting eigenvectors (i.e., here the second eigenvector) are not in this range. We can also characterize eigenfunctions $f$ and corresponding eigenvalues $\lambda$ by
\begin{eqnarray*}
& & Uf(x) =\lambda f(x) \ \  \forall x \in [0,1] \\
& \Leftrightarrow & Md f(x) -Sf(x)=\lambda f(x) \\
& \Leftrightarrow &  d(x)f(x)-\int_0^1 k(x,y) f(y) d(y)=\lambda f(x)\\
& \Leftrightarrow & f(x)(-x^2+x+1/2)-\int_0^1 (1-|x-y|) f(y) d(y)=\lambda f(x)
\end{eqnarray*}
Differentiating twice we get
\begin{equation}
\label{eqdiffFiedler}
f''(x)(1/2-\lambda+x-x^2)+2f'(x)(1-2x)=0.
\end{equation}
The asymptotic expression for the Fiedler vector is then a solution to this differential equation, with $\lambda < 0.5$. Let $\gamma_1$ and $\gamma_2$ be the roots of $(1/2-\lambda+x-x^2)$ (with $\gamma_1<\gamma_2$). We can suppose that $x\in(\gamma_1, \gamma_2)$ since the degree function is nonnegative.  Simple calculations show that 
\[
f'(x)=\frac{A}{(x-\gamma_1)^2(x-\gamma_2)^2}
\]
is solution to~\eqref{eqdiffFiedler}, where $A$ is a constant.
Now we note that
\BEAS
\frac{1}{(x-\gamma_1)^2(x-\gamma_2)^2}&=&\frac{1}{(\gamma_1-\gamma_2)^2(\gamma_2-x)^2}+\frac{1}{(\gamma_1-\gamma_2)^2(\gamma_1-x)^2}\\
&&-\frac{2}{(\gamma_1-\gamma_2)^3(\gamma_2-x)}+\frac{2}{(\gamma_1-\gamma_2)^3(\gamma_1-x)}.
\EEAS
We deduce that the solution $f$ to ~\eqref{eqdiffFiedler} satisfies
\[
f(x)=B+\frac{A}{(\gamma_1-\gamma_2)^2}\left(\frac{1}{\gamma_1-x}+\frac{1}{\gamma_2-x}\right)-\frac{2A}{(\gamma_1-\gamma_2)^3}\left(\log(x-\gamma_1) - \log(\gamma_2-x) \right),
\]
where $A$ and $B$ are two constants.
Since $f$ is orthogonal to the unitary function for $x \in (0,1)$, we must have $f(1/2)=0$, hence $B$=0 (we use the fact that $\gamma_1=\frac{1-\sqrt{1+4\alpha}}{2}$ and $\gamma_2=\frac{1+\sqrt{1+4\alpha}}{2}$, where $\alpha=1/2-\lambda$).

As shown in figure~\ref{fig:asFiedler} , the asymptotic expression for the Fiedler vector is very accurate numerically, even for small values of $n$. The asymptotic Fiedler value is also very accurate (2 digits precision for $n=10$, once normalized by $n^2$).

\begin{figure}[t]
\begin{center}
\begin{tabular}{cc}
\includegraphics[scale=0.4]{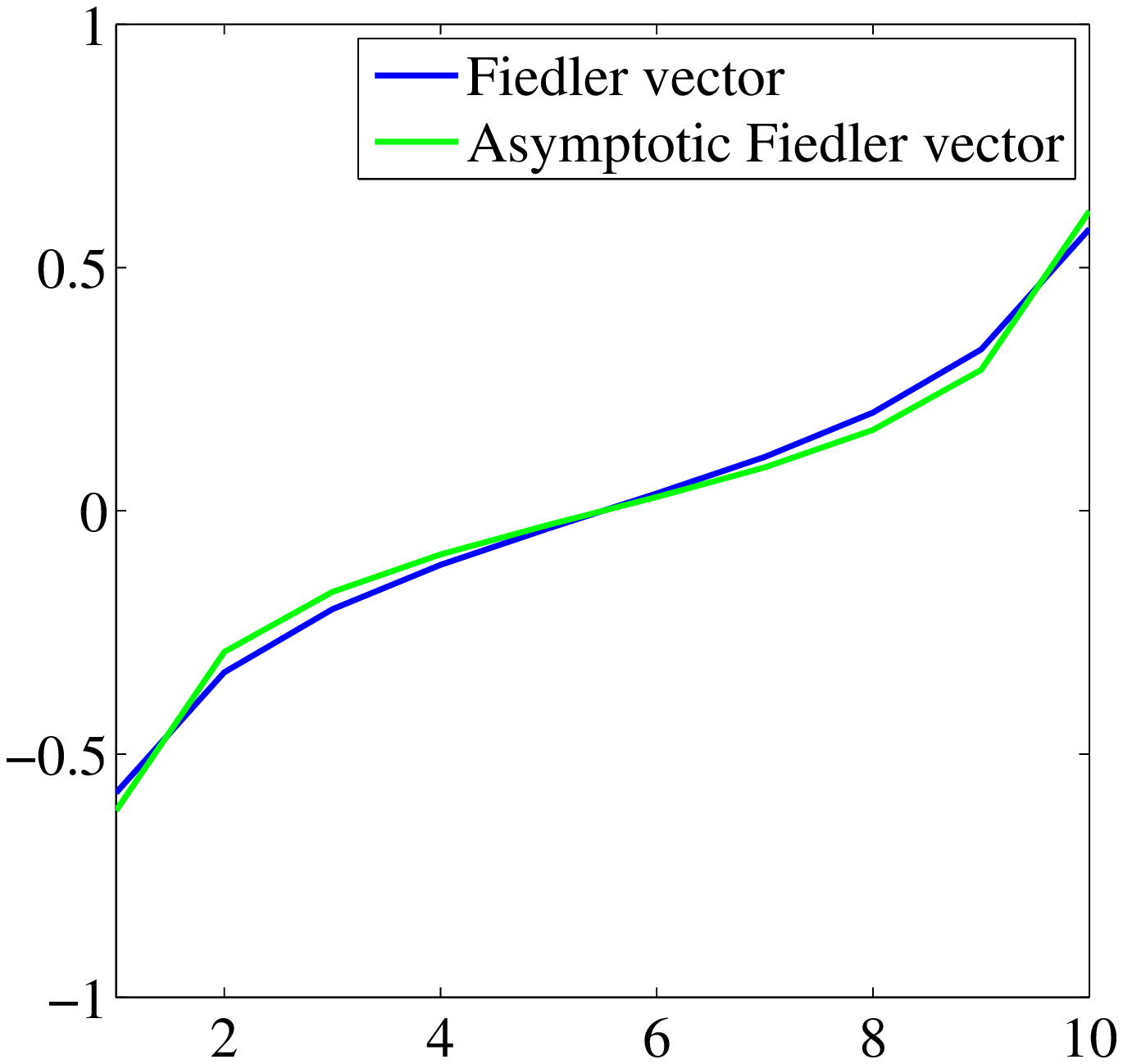}&
\includegraphics[scale=0.4]{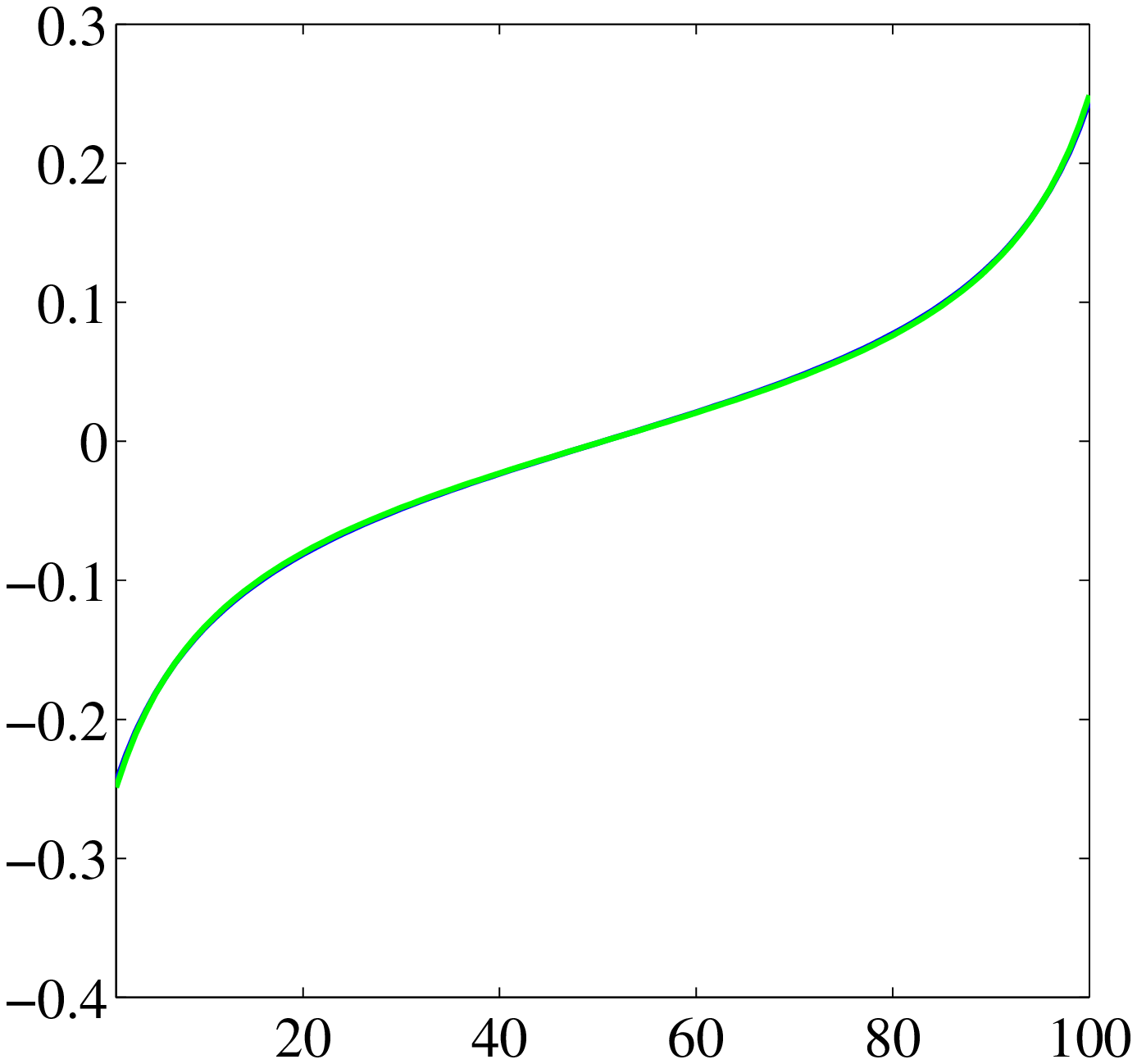}
\end{tabular}
\caption{Comparison between the asymptotic analytical expression of the Fiedler vector and the numeric values obtained from eigenvalue decomposition, for $n=10$ (\emph{left}) and $n=100$ (\emph{right}).
\label{fig:asFiedler}}
\end{center}
\end{figure}

\subsubsection{Bounding the eigengap}

We now give two simple propositions on the Fiedler value and the third eigenvalue of the Laplacian matrix, which enable us to bound the eigengap between the second and the third eigenvalues.

\begin{proposition}
\label{prop:FiedlerUppBound}
Given all comparisons indexed by their true ranking, let $\lambda_2$ be the Fiedler value of $S^{\mathrm{match}}$, we have
\[
\lambda_2\leq\frac{2}{5}(n^2+1).
\]
\end{proposition}
\begin{proof}
Consider the vector $x$ whose elements are uniformly spaced and such that $x^T\ones=0$ and $\|x\|_2=1$. $x$ is a feasible solution to the Fiedler eigenvalue minimization problem. Therefore,
\[
 \lambda_2 \leq x^TLx.
\]
Simple calculations give $x^TLx=\frac{2}{5}(n^2+1)$.
\end{proof}

Numerically the bound is very close to the true Fiedler value: $\lambda_2/n^2 \approx 0.39$ and $2/5=0.4$.

\begin{proposition}
\label{prop:thirdEigenVec}
Given all comparisons indexed by their true ranking, the vector $v=[\alpha, -\beta, \ldots, -\beta, \alpha]^T$ where $\alpha$ and $\beta$ are such that $v^T\ones=0$ and $\|v\|_2=1$ is an eigenvector of the Laplacian matrix $L$ of $S^{\mathrm{match}}$ The corresponding eigenvalue is $\lambda =n(n+1)/2$.
\end{proposition}
\begin{proof}
Check that $Lv=\lambda v$.
\end{proof}

\subsection{Other choices of similarities}

The results in this paper shows that forming a similarity matrix (R-matrix) from pairwise preferences will produce a valid ranking algorithm. In what follows, we detail a few options extending the results of Section~\ref{sec:simMatrix}.

\subsubsection{Cardinal comparisons}
When input comparisons take continuous values between -1 and 1, several choice of similarities can be made.
First possibility is to use $S^{\rm glm}$. An other option is to directly provide $1- \rm{abs} (C)$ as a similarity to \ref{alg:SerialRank}. This option has a much better computational cost.

\subsubsection{Adjusting contrast in $S^{\rm match}$}
Instead of providing $S^{\rm match}$ to \ref{alg:SerialRank}, we can change the ``contrast" of the similarity, i.e., take the similarity whose elements are powers of the elements of  $S^{\rm match}$.
\[
S^{\rm contrast}_{i,j}=(S^{\rm match}_{i,j})^{\alpha}.
\]
This construction gives slightly better results in terms of robustness to noise on synthetic datasets.

\subsection{Hierarchical Ranking}
In a large dataset, the goal may be to rank only a subset of top items. In this case, we can first perform spectral ranking, then refine the ranking of the top set of items using either the \ref{alg:SerialRank} algorithm on the top comparison submatrix, or another seriation algorithm such as the convex relaxation in \citep{Foge13}. This last method also allows us to solve semi-supervised ranking problems, given additional information on the structure of the solution.

\subsection*{Acknowledgements}
AA is at CNRS, at the D\'epartement d'Informatique at \'Ecole Normale Sup\'erieure in Paris, INRIA - Sierra team, PSL Research University. The authors would like to acknowledge support from a starting grant from the European Research Council (ERC project SIPA), the MSR-Inria Joint Centre, as well as support from the chaire {\em \'Economie des nouvelles donn\'ees}, the {\em data science} joint research initiative with the {\em fonds AXA pour la recherche} and a gift from Soci\'et\'e G\'en\'erale Cross Asset Quantitative Research.

\end{document}